%% file: main.tex
\documentclass{article}

\usepackage[preprint]{corl_2026} 

\usepackage[utf8]{inputenc} 
\usepackage[T1]{fontenc}    
\usepackage{url}            
\usepackage{booktabs}       
\usepackage{nicefrac}       
\usepackage{microtype}      
\usepackage{xcolor}         

\usepackage{graphicx}
\usepackage{xurl}
\usepackage{amsmath}
\usepackage{amssymb}
\usepackage{mathtools}
\usepackage{amsthm}
\usepackage{amsfonts}
\usepackage{bm}
\usepackage{multirow}
\usepackage{adjustbox}
\usepackage{threeparttable}
\usepackage{bbding}
\usepackage{colortbl}
\usepackage{array}
\usepackage{enumitem}
\usepackage{caption}
\usepackage{mdframed}
\usepackage[most]{tcolorbox}
\usepackage{epigraph}
\usepackage{arydshln}
\usepackage{titletoc}
\usepackage{pifont}
\usepackage{soul}

\usepackage{tikz}
\usetikzlibrary{positioning,arrows.meta,fit,backgrounds,calc}

\usepackage[capitalize]{cleveref}
\crefname{section}{Sec.}{Secs.}
\Crefname{section}{Section}{Sections}
\crefname{appendix}{appendix}{appendices}
\Crefname{appendix}{Appendix}{Appendices}
\Crefname{table}{Table}{Tables}
\crefname{table}{Tab.}{Tabs.}
\theoremstyle{plain}
\newtheorem{theorem}{Theorem}[section]
\newtheorem{proposition}[theorem]{Proposition}

\theoremstyle{definition}

\theoremstyle{remark}

\setlist[itemize]{
  leftmargin=1.2em,
  itemsep=0pt,
  topsep=0pt,
  parsep=0pt
}

\usepackage{algpseudocode}
\usepackage{algorithm}
\usepackage{tabularx}
\usepackage{tabularray}
\makeatletter
\newcommand{\multiline}[1]{%
  \begin{tabularx}{\dimexpr\linewidth-\ALG@thistlm}[t]{@{}X@{}}
    #1
  \end{tabularx}
}
\makeatother

\definecolor{royalblue}{rgb}{0.25, 0.41, 0.88}
\definecolor{mediumgray}{rgb}{0.5,0.5,0.5}
\definecolor{crimsonred}{RGB}{180, 30, 40}

\definecolor{gray}{HTML}{b2b2b2}
\definecolor{lightgray}{HTML}{cdcdcd}
\definecolor{lightblue}{HTML}{e6ebf2}
\definecolor{lightpurple}{HTML}{EEE5F4}

\definecolor{darkblue}{HTML}{718fc1}
\definecolor{darkpurple}{HTML}{c87ea5}
\definecolor{darkgray}{HTML}{555555}

\newcommand{\cmark}{\textcolor{darkgray}{\ding{51}}}
\newcommand{\xmark}{\textcolor{gray}{\ding{55}}}
\newcommand{\secondbest}[1]{\underline{#1}}
\newcommand{\pmsd}[1]{\,{\color{mediumgray}{\tiny $\pm$\,#1}}}

\newcommand{\tablestyle}[2]{\setlength{\tabcolsep}{#1}\renewcommand{\arraystretch}{#2}\centering\footnotesize}

\newmdenv[linecolor=light-gray,backgroundcolor=light-gray]{graybox}

\input{math_commands.tex}

\hypersetup{
    colorlinks=true,
    citecolor=royalblue, 
    linkcolor=royalblue, 
    filecolor=magenta,      
    urlcolor=royalblue
}

\title{CRAFT: Counterfactual-to-Interactive Reinforcement Fine-Tuning for Driving Policies}

\author{%
Keyu Chen$^{1}$~
Nanfei Ye$^{2}$~
Yida Wang$^{2}$~
Wenchao Sun$^{1}$~
Danqi Zhao$^{1}$~
Hao Cheng$^{1}$~
Sifa Zheng$^{1}$ \\
[2mm]
$^1$~School of Vehicle and Mobility, Tsinghua University ~
$^2$~Li Auto Inc~
}

\begin{document}

\maketitle

\begin{abstract}
Open-loop imitation learning has advanced modern autonomous driving policy architectures, but closed-loop deployment remains vulnerable to policy-induced distribution shift. Existing post-training paradigms exhibit fundamental trade-offs: closed-loop RL fine-tuning provides grounded feedback from executed actions but is constrained by the sparsity of informative events, whereas counterfactual fine-tuning provides dense supervision over candidate futures but inherits bias from imperfect future estimates. We introduce \textbf{Counterfactual-to-Interactive Reinforcement Fine-Tuning (CRAFT)}, an on-policy framework that formulates closed-loop post-training as proxy-residual optimization. CRAFT uses group-normalized counterfactual advantages as a dense proxy for real closed-loop advantages and aligns this proxy with the closed-loop world through grounded residual correction from interaction-critical events. To stabilize adaptation, CRAFT regularizes the online policy toward an EMA teacher via asymmetric KL self-distillation. Theoretically, CRAFT decomposes the real closed-loop policy gradient into proxy and residual terms under the same visited-state distribution, reducing residual variance with an aligned proxy while mitigating proxy bias through grounded residual approximation. Empirically, CRAFT achieves the strongest closed-loop gains on Bench2Drive across hierarchical planning, vision-language-action, and vocabulary-scoring architectures. Ablations, scaling behavior, stability analyses, and transfer results further validate the complementary roles of dense counterfactual proxy and grounded residual correction. Project page: \url{https://currychen77.github.io/CRAFT}.
\end{abstract}

\input{tex/1_introduction}
\input{tex/2_relatedwork}
\input{tex/3_methodology}
\input{tex/4_experiment}
\input{tex/5_conclusions}

\newpage

{\small
\bibliography{cite}
}

\newpage

\appendix
\crefalias{section}{appendix}
\crefalias{subsection}{appendix}
\crefalias{subsubsection}{appendix}

\input{tex/6_appendix}

\end{document}

%% file: math_commands.tex
\newcommand{\Var}{\mathrm{Var}}
\newcommand{\Cov}{\mathrm{Cov}}

%% file: tex/1_introduction.tex
\section{Introduction}
\label{sec:intro}

Modern autonomous driving policies have rapidly evolved beyond modular perception--planning pipelines. Representative frameworks include unified end-to-end models~\cite{hu2023_uniad,jiang2023vad,sun2025sparsedrive}, vision-language-action models~\cite{peng2025counterfactualvla,dang2026drivefine}, and world action models~\cite{xia2025drivelaw,li2025drivevlaw0}. While these architectures expand the capability of driving policies, their training remains largely governed by an open-loop imitation learning (IL) paradigm. This paradigm often leads to distribution shift due to the inherent discrepancy between open-loop training and closed-loop deployment. Therefore, further improving the closed-loop performance of pre-trained driving policies remains critical for safe and reliable deployment.

To address the challenges of distribution shift, two distinct post-training paradigms have emerged. As illustrated in \Cref{fig:intro}, closed-loop RL fine-tunes pre-trained driving policies within interactive simulators to obtain grounded feedback from executed actions. However, the inherent sparsity of informative events across diverse scenarios often constrains the training efficiency of pure closed-loop RL. Conversely, counterfactual fine-tuning offers an alternative by evaluating multiple candidate futures from the same visited state, transforming these relative performance differences into dense supervisory signals through group-relative objectives, such as GRPO~\cite{shao2024grpo}. Nevertheless, counterfactual estimates of the future remain inherently biased, as the counterfactual world cannot faithfully capture the interaction dynamics of the real closed-loop world.

\begin{figure}[t]
    \centering
    \includegraphics[width=0.98\textwidth]{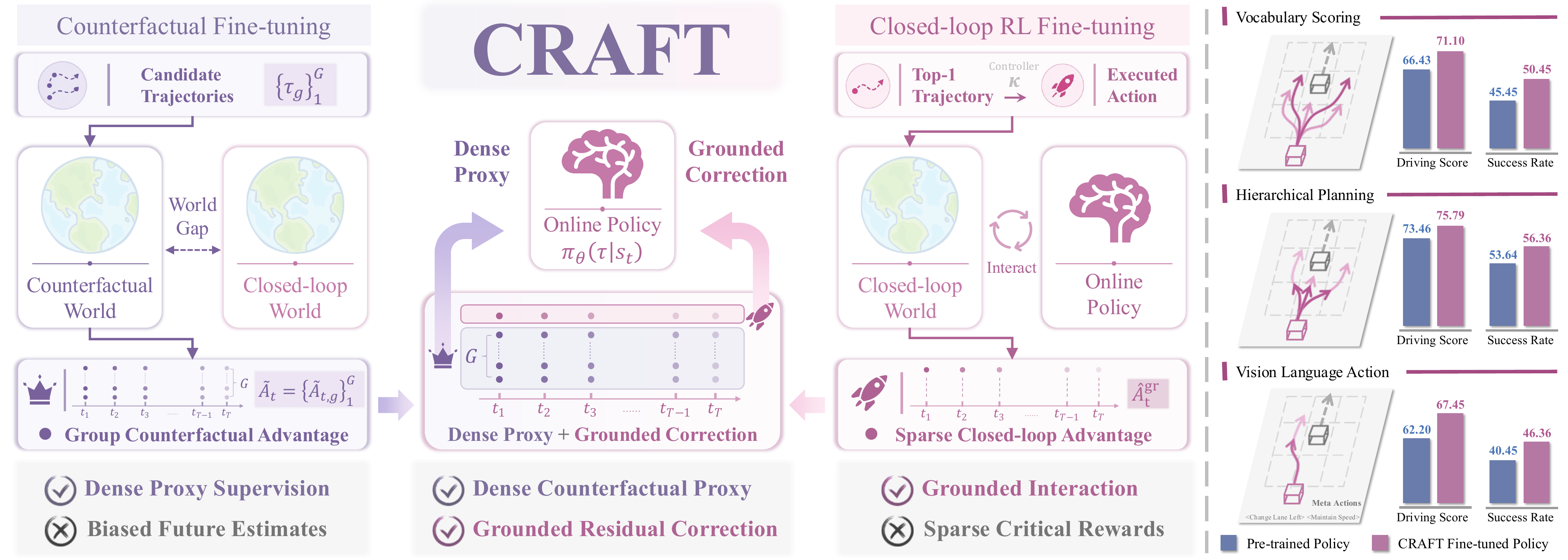}
    \caption{Two existing post-training paradigms exhibit complementary trade-offs. Closed-loop RL fine-tuning provides grounded but sparse feedback from executed actions, whereas counterfactual fine-tuning provides dense supervision from biased future estimates. CRAFT combines their complementary strengths through a dense counterfactual proxy and grounded residual correction.}
    \vspace{8pt}
    \label{fig:intro}
\end{figure}

The complementary strengths of these two paradigms motivate a natural question: can counterfactual supervision serve as a dense proxy, while closed-loop interaction corrects its residual mismatch with the real closed-loop world? CRAFT addresses this question by assigning distinct optimization roles to the two signals. Counterfactual supervision serves as a dense proxy over states visited by the current policy, whereas closed-loop interaction provides grounded residual correction to align this proxy with the real closed-loop world.

Building on this proxy-residual decomposition, we introduce Counterfactual-to-Interactive Reinforcement Fine-Tuning (CRAFT), an on-policy post-training framework for pre-trained driving policies. As illustrated in \Cref{fig:intro}, CRAFT uses the counterfactual world to evaluate groups of candidates from states visited by the online policy and converts their efficiency, safety, and rule-compliance returns into group-normalized counterfactual advantages. These advantages serve as a dense proxy for real closed-loop advantages. CRAFT then derives grounded residual correction from interaction-critical traffic events observed in executed closed-loop rollouts, aligning the proxy with the real closed-loop world through a value-free dual-clipped update~\cite{ye2020dual_clip, gao2021learning}. To stabilize adaptation, asymmetric-KL self-distillation regularizes the online policy toward an EMA teacher initialized from the pre-trained policy, preserving reliable behaviors while allowing interaction-driven correction.

Theoretically, CRAFT decomposes the real closed-loop policy gradient into proxy and residual terms under the same state distribution, reducing residual variance with an aligned proxy while mitigating proxy bias through grounded residual approximation. Dual clipping further keeps rare critical events informative without dominating the update. Empirically, CRAFT achieves the largest closed-loop gains on Bench2Drive across policy architectures, including hierarchical planning, vision-language-action, and vocabulary-scoring. Ablations verify the complementary roles of each component, with the counterfactual proxy providing broad optimization signals, grounded residual correction delivering substantial closed-loop gains on interaction-sensitive metrics, and self-distillation stabilizing adaptation around reliable pre-trained behavior. Scaling and stability analyses further demonstrate that, with increasing fine-tuning data, CRAFT delivers more consistent gains and better training stability than baselines, while transfer results indicate partial robustness beyond the training distribution. Our main contributions are summarized as follows:

\begin{itemize}[leftmargin=1.2em,itemsep=0.0em,topsep=0.1em]
\item We formulate closed-loop post-training for pre-trained driving policies through a proxy-residual gradient decomposition, revealing the trade-off between dense but biased counterfactual supervision and sparse but grounded closed-loop interaction under the same visited-state distribution.
\item We propose CRAFT, an on-policy post-training framework that turns counterfactual supervision into a dense proxy for real closed-loop advantages and uses interaction-critical closed-loop events as grounded residual correction, stabilized by self-distillation from the pre-trained policy.
\item We substantiate the proxy-residual design theoretically and empirically. CRAFT reduces residual variance with an aligned proxy while mitigating proxy bias through grounded residual approximation, and achieves consistent closed-loop gains on Bench2Drive across diverse policy architectures, including hierarchical planning, vision-language-action, and vocabulary-scoring.
\end{itemize}

%% file: tex/2_relatedwork.tex
\section{Related Work}
\label{sec:relatedwork}

\vspace{-6pt}
\paragraph{End-to-End Autonomous Driving Policy.}
Progress in end-to-end driving has been driven largely by advances in model architecture. Diffusion-based models~\cite{liao2025diffusiondrive,liu2025guideflow,zheng2025resad,zheng2026HDP}, vision-language-action models~\cite{fu2025orion,renz2025simlingo,zhou2025autovla,wang2026linkvla}, and trajectory-vocabulary scoring methods~\cite{li2024hydramdp,yao2026drivesuprim,li2025GTRS,sun2026sparsedrivev2} all improve policy expressiveness. Yet their training remains largely imitation-based, with policies trained near expert states and then expected to remain reliable after their own actions shift future observations. This assumption is especially fragile in closed-loop deployment, where small planning errors can move the ego vehicle outside the training distribution. 

\vspace{-10pt}
\paragraph{Closed-Loop RL in Autonomous Driving.}
Closed-loop RL mitigates distribution shift by optimizing policies on executed rollouts, often with policy-gradient methods such as PPO~\cite{PPO} in CARLA~\cite{fu2025minddrive,jaeger2025carl} or reconstructed simulators~\cite{gao2025rad,gao2026rad2,ni2025recondreamer}. While grounded in real interaction, these methods remain sample-inefficient because informative events are rare, requiring extensive rollouts to obtain useful gradient estimates. Recent approaches improve efficiency through VLM-based reward densification~\cite{huang2026drivevlm_rl} or world-model rollouts~\cite{li2024think2drive,yang2025raw2drive}, but their supervision depends on evaluator or model fidelity. CRAFT instead uses closed-loop interaction only to correct a dense counterfactual proxy, preserving grounded feedback without estimating full advantages from sparse rollouts.

\vspace{-10pt}
\paragraph{Counterfactual Post-Training in Autonomous Driving.}
Counterfactual post-training improves learning efficiency through dense local comparisons on real visited states~\cite{liu2026R2SE}, making preference-based and group-relative objectives attractive~\cite{shao2024grpo,rafailov2023dpo}. Recent methods combine PDM Score~\cite{li2025recogdrive,zou2025diffusiondrivev2,yasarla2026gero,shang2026dynvla,liang2025DIPOLE,chen2026Curious_VLA} with GRPO in NAVSIM~\cite{dauner2024navsimv1} for policy fine-tuning, or use synthetic counterfactual data~\cite{lin2025MPA,yan2025ad_r1} for offline post-training. A key limitation lies in the rollout condition, since counterfactual futures generated offline or in a non-reactive simulator do not capture true closed-loop interaction with the environment. CRAFT uses counterfactual futures as a dense proxy on real visited states, while grounded residual correction refines the proxy with feedback from true closed-loop rollouts.

%% file: tex/3_methodology.tex
\section{Methodology}
\label{sec:method}

CRAFT is a general paradigm for closed-loop post-training of supervised driving policies. The target is to improve closed-loop performance under the on-policy state distribution. Existing post-training paradigms present a fundamental trade-off: closed-loop RL is physically grounded but relies on sparse, high-variance interaction feedback, whereas counterfactual optimization provides dense supervision but suffers from biased future estimates. CRAFT addresses this trade-off with two complementary components: a dense counterfactual proxy that provides trajectory-level supervision via candidate evaluation, and a grounded residual correction that compensates for proxy bias using executed closed-loop feedback. To preserve pre-trained behavior during fine-tuning, CRAFT further applies on-policy self-distillation with an EMA teacher. The framework is illustrated in \Cref{fig:teaser}.

\begin{figure}[t]
    \centering
    \includegraphics[width=\textwidth]{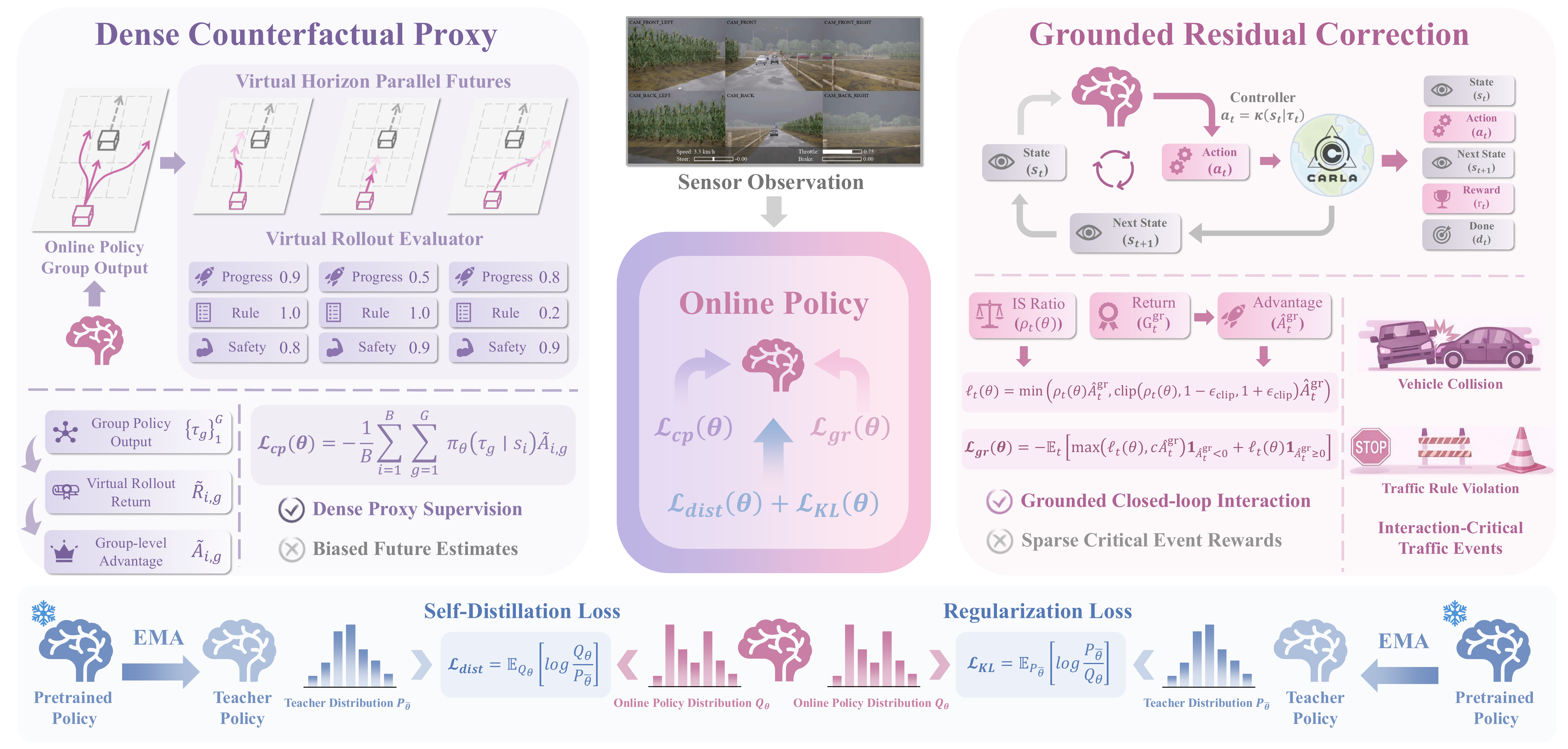}
\caption{CRAFT combines dense counterfactual proxy supervision, grounded residual correction from real closed-loop interaction, and on-policy self-distillation toward reliable pre-trained behavior.}
    \vspace{8pt}
    \label{fig:teaser}
\end{figure}

\subsection{Real Closed-Loop Objective and Proxy–Residual Decomposition}
\label{subsec:unified}

We first formalize the real closed-loop objective that CRAFT aims to improve. We then show that CRAFT optimizes it through two complementary terms defined on the same on-policy state distribution. The key idea is to use dense counterfactual supervision as a proxy for the full closed-loop learning signal, while using grounded interaction feedback to correct the residual mismatch.

Let $\mathcal{M}_{\mathrm{real}}=(\mathcal{S},\mathcal{A},P_{\mathrm{real}},r_{\mathrm{real}},\gamma)$ denote the real closed-loop MDP, with state space $\mathcal{S}$, low-level action space $\mathcal{A}$, transition function $P_{\mathrm{real}}$, reward function $r_{\mathrm{real}}$, and discount factor $\gamma\in(0,1)$. The policy scores candidate trajectories $\tau\in\mathcal{T}$ by $\pi_\theta(\tau\mid s)$, and the controller $\kappa$ tracks the selected trajectory $\tau_t$ by executing $a_t=\kappa(s_t,\tau_t)$, with rewards evaluated after the induced controller action. Let $d_{\mathrm{real}}^{\pi_\theta}$ be the discounted state-visitation measure. For closed-loop rollout $\xi$, the objective is
\begin{equation}
J_{\mathrm{real}}(\pi_\theta)
=
\mathbb{E}_{\xi \sim (\pi_\theta,\kappa,P_{\mathrm{real}})}
\Big[
\sum_{t=0}^{\infty}\gamma^t r_{\mathrm{real}}(s_t,a_t,s_{t+1})
\Big],
\end{equation}
whose policy gradient is
\begin{equation}
\nabla_\theta J_{\mathrm{real}}(\pi_\theta)
=
\mathbb{E}_{s\sim d_{\mathrm{real}}^{\pi_\theta},\,\tau\sim \pi_\theta(\cdot\mid s)}
\Big[
\nabla_\theta \log \pi_\theta(\tau\mid s)\,
A^{\mathrm{real}}_{\pi_\theta}(s,\tau)
\Big].
\label{eq:real_pg}
\end{equation}
Here $A^{\mathrm{real}}_{\pi_\theta}(s,\tau)$ denotes the real advantage of selecting $\tau$ and executing its induced controller action $a$. Grounded rollouts alone are sample-inefficient because rare critical events provide weak learning signal for most visited states, while the resulting failures dominate return variance. Consequently, value estimation is brittle and value-free optimization suffers from high-variance updates.

CRAFT introduces a dense counterfactual proxy on real visited states. For each $s\sim d_{\mathrm{real}}^{\pi_\theta}$, the counterfactual world assigns each candidate trajectory $\tau$ a proxy value $\Phi_{\pi_\theta}(s,\tau)$, such as a short-horizon rollout score or a group-relative preference. As $\Phi_{\pi_\theta}$ only approximates the real advantage, its residual mismatch is
$\Delta_{\pi_\theta}(s,\tau)\triangleq A^{\mathrm{real}}_{\pi_\theta}(s,\tau)-\Phi_{\pi_\theta}(s,\tau)$.

\begin{proposition}[Exact Counterfactual-Residual Decomposition]
\label{prop:exact_decomposition}
For any integrable dense counterfactual proxy $\Phi_{\pi_\theta}(s,\tau)$ defined on real visited states, the closed-loop gradient decomposes as
\begin{equation}
\nabla_\theta J_{\mathrm{real}}(\pi_\theta)
=
\mathbb{E}
\Big[
\nabla_\theta \log \pi_\theta(\tau\mid s)\,
\Phi_{\pi_\theta}(s,\tau)
\Big]
+
\mathbb{E}
\Big[
\nabla_\theta \log \pi_\theta(\tau\mid s)\,
\Delta_{\pi_\theta}(s,\tau)
\Big],
\label{eq:exact_decomposition}
\end{equation}
where both expectations are over $s\sim d_{\mathrm{real}}^{\pi_\theta}$ and $\tau\sim\pi_\theta(\cdot\mid s)$. The proof is provided in \Cref{app:proof_exact_decomposition}.
\end{proposition}
\noindent The proxy term captures the dense component of the real objective, while the residual term accounts for the mismatch between the proxy and real advantage. Since both terms are defined on real visited states, the decomposition introduces no state-distribution mismatch.

The decomposition can reduce variance when the proxy is aligned with the real advantage, since the residual term only needs to estimate the remaining mismatch through grounded rollouts.

\begin{proposition}[Conditional Variance Reduction]
\label{prop:variance_reduction}
Fix a state $s$ and let $Y(s,\tau)=A^{\mathrm{real}}_{\pi_\theta}(s,\tau)$ and $C(s,\tau)=\Phi_{\pi_\theta}(s,\tau)$. For $R_\alpha(s,\tau)=Y(s,\tau)-\alpha C(s,\tau)$ with $\tau\sim\pi_\theta(\cdot\mid s)$ and $\Var_\tau[C(s,\tau)]>0$,
\begin{equation}
\Var_\tau[R_\alpha(s,\tau)]
=
\Var_\tau[Y(s,\tau)]
+
\alpha^2\Var_\tau[C(s,\tau)]
-2\alpha\Cov_\tau[Y(s,\tau),C(s,\tau)],
\label{eq:variance_identity}
\end{equation}
and the minimizing coefficient is $\alpha^\star(s)=\Cov_\tau[Y(s,\tau),C(s,\tau)]/\Var_\tau[C(s,\tau)]$. In particular, if $\Cov_\tau[Y(s,\tau),C(s,\tau)]>\frac12\Var_\tau[C(s,\tau)]$, then $\Var_\tau[Y(s,\tau)-C(s,\tau)]<\Var_\tau[Y(s,\tau)]$. If this condition holds for $d_{\mathrm{real}}^{\pi_\theta}$-almost every $s$, the same strict inequality holds after averaging over $s$.
\end{proposition}
\noindent The proof is in \Cref{app:proof_variance_reduction}. Importantly, the proxy need not be unbiased, since residual variance reduction only requires the proxy to be sufficiently aligned with the real advantage.

CRAFT instantiates this decomposition with the surrogate objective:
\begin{equation}
\max_{\pi_\theta}\;
\underbrace{
\mathbb{E}_{s \sim d_{\mathrm{real}}^{\pi_\theta}}
\Bigl[\mathcal{R}_{\mathrm{cp}}(s;\pi_\theta)
\;\vphantom{\sum_t r_t^{\mathrm{gr}}}\Bigr]
}_{\text{\small Dense Counterfactual Proxy}}
+
\underbrace{
\mathbb{E}_{\xi \sim (\pi_\theta,\kappa,P_{\mathrm{real}})}
\Bigl[\sum_{t}\gamma^t r_t^{\mathrm{gr}}\Bigr]
}_{\text{\small Grounded Residual Correction}},
\label{eq:proxy_residual_objective}
\end{equation}
where the two objective components are defined in \Cref{subsec:cp,subsec:gr}.

\subsection{Dense Counterfactual Proxy}
\label{subsec:cp}

For each on-policy state $s_i$ collected from closed-loop interaction, the counterfactual world evaluates a group of $G$ candidate trajectories $\{\tau_g\}_{g=1}^G$ over a finite virtual horizon. Evaluation criteria covering efficiency, safety, and rule compliance yield a scalar return $\smash{\tilde R_{i,g}}$ for each candidate. Implementation details of the counterfactual world and reward design are provided in \Cref{app:counterfactual_engine,app:reward_design}. We convert these returns into local relative scores by normalizing within each candidate group:
\begin{equation}
\tilde{A}_{i,g}
=
\frac{\tilde{R}_{i,g}-\mu_i}{\sigma_i+\varepsilon_{\mathrm{norm}}},
\qquad
\mu_i=\frac{1}{G}\sum_{g=1}^{G}\tilde{R}_{i,g},
\label{eq:group_adv}
\end{equation}
where $\sigma_i$ is the empirical standard deviation of $\{\tilde R_{i,g}\}_{g=1}^G$. In CRAFT, this normalized score is the proxy value $\smash{\Phi_{\pi_\theta}(s_i,\tau_g)=\tilde A_{i,g}}$. Since fine-tuning is performed in a discrete action space, the counterfactual objective takes the exact expectation over candidate trajectories:
\begin{equation}
\mathcal{L}_{\mathrm{cp}}
=
-\frac{1}{B}\sum_{i=1}^{B}\sum_{g=1}^{G}
\pi_\theta(\tau_g \mid s_i)\,\tilde{A}_{i,g}.
\label{eq:lcp}
\end{equation}
This objective instantiates the dense counterfactual proxy term in \Cref{eq:proxy_residual_objective}, where $\mathcal{R}_{\mathrm{cp}}(s_i;\pi_\theta)=\sum_{g=1}^{G}\pi_\theta(\tau_g\mid s_i)\tilde A_{i,g}$. It provides dense group-wise comparative supervision at each on-policy state by evaluating candidate trajectories in a counterfactual world designed to approximate the real closed-loop world. However, unrealistic traffic interactions and dynamics biases still introduce a discrepancy between the proxy advantage $\tilde A_{i,g}$ and the real advantage $A^{\mathrm{real}}_{\pi_\theta}(s_i,\tau_g)$, which motivates the grounded residual correction.

\subsection{Grounded Residual Correction}
\label{subsec:gr}

Rather than estimating the exact residual in \Cref{eq:exact_decomposition}, CRAFT approximates the interaction-critical residual component using grounded feedback from traffic events that are not fully captured by the counterfactual proxy, such as collisions and traffic-rule violations. For a selected trajectory $\tau_t$, the controller executes $a_t=\kappa(s_t,\tau_t)$, after which the event-driven corrective reward defined in \Cref{app:reward_design} is computed from the closed-loop transition and converted into bounded corrective advantages $\hat A^{\mathrm{gr}}_t$, as detailed in \Cref{app:objective_details}.

Let $\rho_t(\theta)=\pi_\theta(\tau_t\mid s_t)/\pi_{\theta_{\mathrm{old}}}(\tau_t\mid s_t)$ denote the importance sampling ratio. We first form
\begin{equation}
\ell_t(\theta)
=
\min\!\Big(
\rho_t(\theta)\hat A^{\mathrm{gr}}_t,\;
\mathrm{clip}(\rho_t(\theta), 1-\epsilon_{\mathrm{clip}}, 1+\epsilon_{\mathrm{clip}})\hat A^{\mathrm{gr}}_t
\Big),
\label{eq:lt}
\end{equation}
and then apply the dual-clipping~\cite{ye2020dual_clip, gao2021learning} with constant $c>1$ to form the grounded residual loss
\begin{equation}
\mathcal{L}_{\mathrm{gr}}
=
-\mathbb{E}_{t\in\mathcal{B}}
\Big[
\max\!\big(\ell_t(\theta), c\hat A^{\mathrm{gr}}_t\big)\mathbf{1}_{\hat A^{\mathrm{gr}}_t<0}
\,+\,
\ell_t(\theta)\mathbf{1}_{\hat A^{\mathrm{gr}}_t\ge 0}
\Big].
\label{eq:lgr}
\end{equation}
Proposition~\ref{prop:dualclip} formalizes the stabilization of dual clipping, which preserves informative gradients from rare catastrophic rollouts while preventing extreme negative events from dominating the update.

To characterize the grounded residual correction under the decomposition in \Cref{eq:exact_decomposition}, let $\widehat{\Delta}_{\pi_\theta}(s,\tau)$ be the residual approximation induced by grounded feedback, and define
\begin{equation}
\widehat g(\theta)
=
\mathbb{E}
\Big[
\nabla_\theta\log\pi_\theta(\tau\mid s)\,
\big(\Phi_{\pi_\theta}(s,\tau)+\widehat{\Delta}_{\pi_\theta}(s,\tau)\big)
\Big].
\label{eq:ghat}
\end{equation}

\begin{theorem}[Gradient Bias under Residual Approximation]
\label{thm:bias_compensation}
Let $g(\theta)$ denote the real gradient in \Cref{eq:exact_decomposition}. If $\mathbb{E}\|\nabla_\theta\log\pi_\theta(\tau\mid s)\|_2^2<\infty$ and $\mathbb{E}|\widehat{\Delta}_{\pi_\theta}(s,\tau)-\Delta_{\pi_\theta}(s,\tau)|^2<\infty$, then
\begin{equation}
\|\widehat g(\theta)-g(\theta)\|_2
\le
\Big(
\mathbb{E}\|\nabla_\theta\log\pi_\theta(\tau\mid s)\|_2^2
\Big)^{1/2}
\Big(
\mathbb{E}|\widehat{\Delta}_{\pi_\theta}(s,\tau)-\Delta_{\pi_\theta}(s,\tau)|^2
\Big)^{1/2}.
\label{eq:bias_bound_residual}
\end{equation}
If $\|\nabla_\theta\log\pi_\theta(\tau\mid s)\|_2\le G_{\nabla}$ almost surely, the first factor reduces to $G_{\nabla}$. Proof in \Cref{app:proof_bias_compensation}.
\end{theorem}
\noindent \Cref{thm:bias_compensation} establishes the role of grounded residual correction in \Cref{eq:proxy_residual_objective}. Grounded feedback captures the dominant mismatch left by the dense counterfactual proxy, focusing correction on the interaction-sensitive part of the real closed-loop objective.

\subsection{On-Policy Self-Distillation}
\label{subsec:distill}

The proxy-residual objective provides dense and grounded learning signals, but unconstrained policy updates may deviate from reliable pre-trained behavior. CRAFT limits this drift through on-policy self-distillation from an EMA teacher $\pi_T$,
\begin{equation}
\theta_T^{(k+1)} \leftarrow m \theta_T^{(k)} + (1-m)\theta^{(k)},\qquad
\theta_T^{(0)}=\theta_{\mathrm{pre}},\quad 0\le m<1.
\label{eq:ema}
\end{equation}
The teacher is initialized from the pre-trained checkpoint, as a conservative anchor early in training, and subsequently tracks the smoothed online policy through EMA. For a visited state $s$, let $U(s,\tau)=\Phi_{\pi_\theta}(s,\tau)+\widehat{\Delta}_{\pi_\theta}(s,\tau)$ denote the proxy-residual advantage estimate. Proposition~\ref{prop:kl_proximal} indicates that policy improvement constrained by reverse KL induces an exponential reweighting of the teacher distribution according to this advantage estimate. This provides a principled update rule that improves the policy under the proxy-residual objective while remaining close to teacher-supported behavior. CRAFT implements this rule with asymmetric KL penalties,
\begin{equation}
\mathcal{L}_{\mathrm{dist}}
=
\mathbb{E}_{s\sim d_{\mathrm{real}}^{\pi_\theta},\,\tau\sim \pi_\theta(\cdot\mid s)}
\left[
\log \frac{\pi_\theta(\tau\mid s)}{\pi_T(\tau\mid s)}
\right], 
\quad
\mathcal{L}_{\mathrm{KL}}
=
\mathbb{E}_{s\sim d_{\mathrm{real}}^{\pi_\theta},\,\tau\sim \pi_T(\cdot\mid s)}
\left[
\log \frac{\pi_T(\tau\mid s)}{\pi_\theta(\tau\mid s)}
\right].
\label{eq:kls}
\end{equation}
The reverse KL term is the primary constraint, with $\beta_r>\beta_f$, and suppresses deviations from reliable pre-trained behavior. The weaker forward KL term preserves teacher-supported alternatives and mitigates collapse to a narrow subset of trajectories favored by the proxy-residual objective.

Overall, CRAFT combines the dense counterfactual proxy, grounded residual correction, and asymmetric self-distillation into the final objective, with Algorithm~\ref{alg:craft} detailing the full training procedure:
\begin{equation}
\mathcal{L}_{\mathrm{all}}
=
\lambda_{\mathrm{cp}}\mathcal{L}_{\mathrm{cp}}
\,+\,
\lambda_{\mathrm{gr}}\mathcal{L}_{\mathrm{gr}}
\,+\,
\beta_r\mathcal{L}_{\mathrm{dist}}
\,+\,
\beta_f\mathcal{L}_{\mathrm{KL}}.
\label{eq:overall_loss}
\end{equation}

%% file: tex/4_experiment.tex
\section{Experiments}
\label{sec:exp}

\subsection{Experimental Setup}

\paragraph{Benchmark}
We conduct experiments on Bench2Drive~\cite{jia2024bench2drive}, a CARLA-based benchmark~\cite{carla} for closed-loop evaluation in autonomous driving. Our setup follows Bench2Drive as the primary protocol, with additional scenario configurations from LEAD~\cite{Nguyen2026lead} and Longest6 V2~\cite{jaeger2025carl}. We also apply several protocol-level refinements to the original scenarios and evaluation protocol to improve efficiency and reduce repeated-run variance. Details are provided in \Cref{app:bench2drive_patch}.
\vspace{-10pt}
\paragraph{Metrics} We adopt the standard Bench2Drive metrics, including Driving Score (DS), Success Rate (SR), and Multi-Ability. DS combines Route Completion (RC) and Infraction Score (IS), where RC measures route progress and IS reflects safety and traffic-rule compliance. SR measures successful scenario completion, while Multi-Ability separately evaluates five advanced urban driving skills.
\vspace{-10pt}
\paragraph{Driving Policies and Fine-tuning Baselines} To demonstrate the generality of CRAFT, we select three representative driving policies from distinct paradigms.
\vspace{-2pt}
\begin{itemize}
\item \textbf{Hierarchical planning.} HiP-AD~\cite{tang2025hipad} uses multi-granularity planning queries together with deformable attention to couple perception and planning within a unified decoder.
\item \textbf{Vision-language-action.} MindDrive~\cite{fu2025minddrive} uses a VLA architecture with separate decision and action experts, translating high-level semantic decisions into low-level trajectory generation.
\item \textbf{Vocabulary scoring.} SparseDriveV2~\cite{sun2026sparsedrivev2} uses sparse scene representation for end-to-end driving, and here we use its vocabulary variant to provide a dense candidate-trajectory space for RL. This property also makes it the default driving policy in our subsequent analysis experiments.
\end{itemize}
We compare CRAFT against the pre-trained policy and three RL fine-tuning baselines under the same setting and evaluation protocol. The baselines differ only in their training objective.
\vspace{-2pt}
\begin{itemize}
\item PPO~\cite{PPO} employs a clipped surrogate objective together with a learned value baseline, making it the standard actor--critic baseline for stable on-policy fine-tuning.
\item REINFORCE++~\cite{hu2025reinforce++} removes the learned critic and uses return-based updates normalized at the batch level, providing a simple critic-free baseline for on-policy fine-tuning.
\item GRPO~\cite{shao2024grpo} replaces a separate critic with relative rewards computed within output groups, making it the closest baseline to the group-based counterfactual update used in CRAFT.
\end{itemize}
\vspace{-10pt}
\paragraph{Implementation Details}
For all policies, we fine-tune only the planning-related modules while freezing pre-trained visual processing. All fine-tuning methods share a CaRL-style reward formulation based on progress and infractions~\cite{jaeger2025carl}. Our RL infrastructure is built on RIFT~\cite{chen2025rift} and supports asynchronous rollout collection and policy updates. Details appear in \Cref{app:trainable_modules,app:counterfactual_engine,app:reward_design,app:objective_details}.
\vspace{-10pt}
\subsection{Main Closed-Loop Results}

\input{tab/main_result}

\begin{figure}[t]
    \centering
    \includegraphics[width=\textwidth]{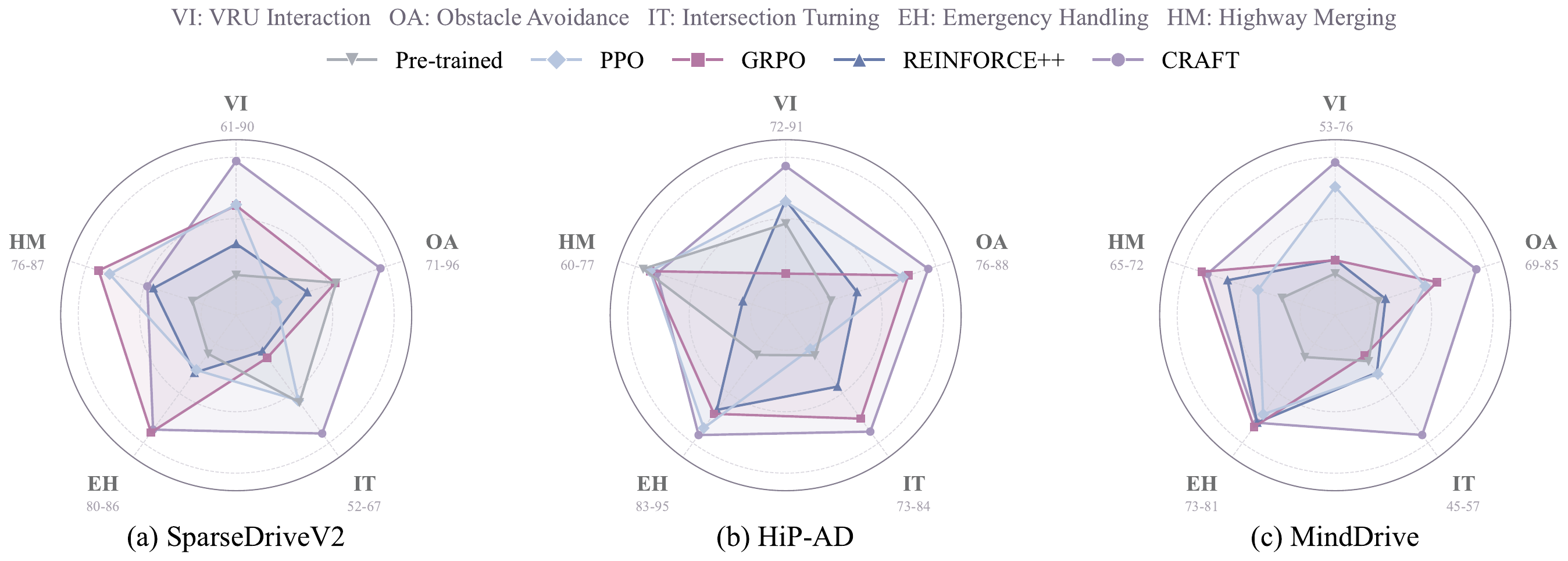}
    \caption{\small \textbf{Fine-grained ability profiles.} Radar plots compare fine-tuning methods across driving policies. Each axis represents a scenario-specific capability, and higher values indicate better performance.}
    \label{fig:ego_ability_radar}
\end{figure}

The main comparison reveals a consistent pattern across policy architectures. CRAFT improves closed-loop metrics without relying on any specific model architecture. This improvement remains evident even for HiP-AD, where the pre-trained policy is already strong, and is larger for MindDrive and SparseDriveV2, where the initial closed-loop performance indicates greater potential for fine-tuning. In \Cref{tab:main_results}, the same ordering holds across hierarchical planning, vision-language-action, and vocabulary-scoring, supporting the claim that CRAFT generalizes across diverse policy architectures.

The comparison also highlights the value of the full CRAFT objective. Even GRPO, the closest proxy-based baseline, underperforms CRAFT across policy architectures. PPO and REINFORCE++ improve specific driving abilities, but their gains are less consistent. This pattern further supports the proxy-residual design of CRAFT, suggesting that dense counterfactual supervision is most effective when refined by grounded residual feedback.

The ability profiles in \Cref{fig:ego_ability_radar} provide a finer-grained view of the aggregate results. CRAFT leads on most abilities, with the largest gains in safety-critical behaviors such as VRU Interaction, Obstacle Avoidance, and Intersection Turning. These abilities depend more strongly on closed-loop interaction and are therefore less likely to be captured by a dense counterfactual proxy alone. This pattern further supports the proxy-residual design of CRAFT, suggesting that dense counterfactual supervision alone is insufficient for safety-critical events and that grounded residual correction is essential.

\input{tab/ablation}

\begin{figure}[t]
    \centering

    \begin{minipage}{\textwidth}
        \centering
        \includegraphics[width=\textwidth]{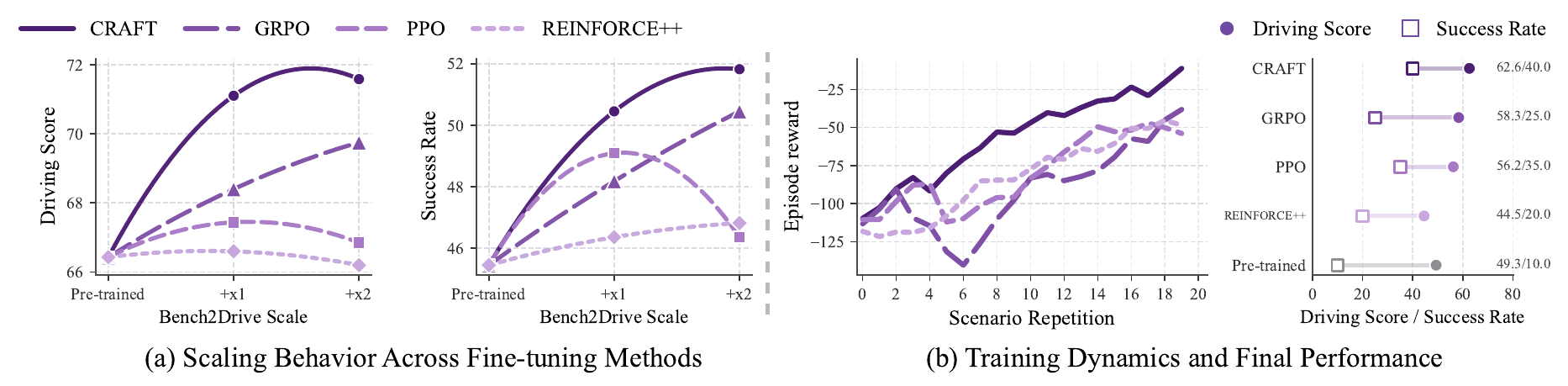}
        \caption{\small \textbf{Scaling behavior and reward dynamics.}
        (a) Closed-loop performance as the Bench2Drive scale increases.
        (b) Reward evolution during RL training on a challenging scenario, with final evaluation performance.}
        \label{fig:scaling}
    \end{minipage}
    \begin{minipage}{\textwidth}
        \centering
        \includegraphics[width=\textwidth]{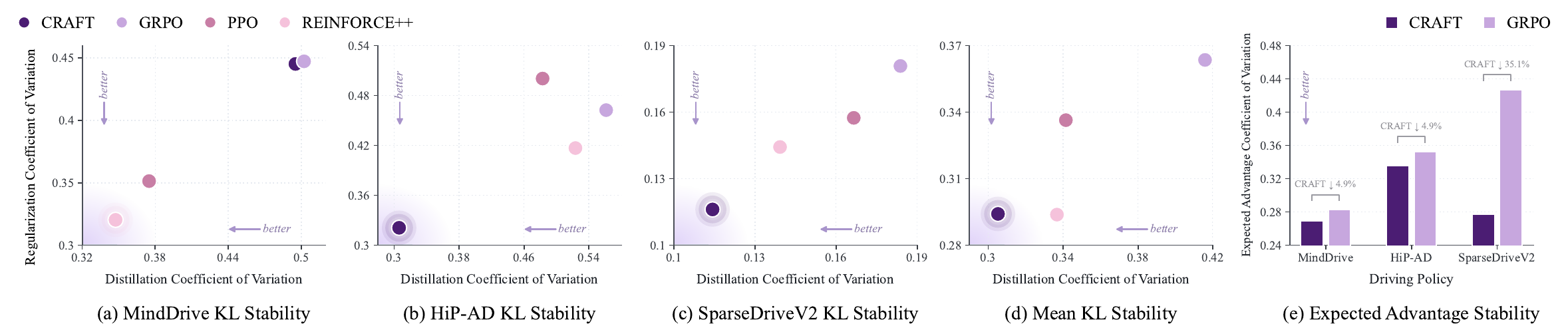}
        \caption{\small \textbf{Training stability.}
        Panels (a)--(d) report the coefficients of variation for KL terms across fine-tuning methods,
        and panel (e) reports the coefficient of variation for the expected advantage across driving policies.}
        \label{fig:train_stability}
    \end{minipage}

\end{figure}

\subsection{Ablation and Generalization Analysis}

The ablation study in \Cref{tab:craft_ablation} clarifies each component of CRAFT. Self-distillation reinforces the dominant modes of the pre-trained policy, while the dense counterfactual proxy drives broad closed-loop gains. With grounded residual correction, the full CRAFT objective achieves the best overall performance, especially on the interaction-related metrics IS and SR. These results support the proxy-residual view, where the dense counterfactual proxy provides the primary optimization signal and grounded residual correction compensates for counterfactual bias through closed-loop interaction.

As shown in \Cref{tab:transfer}, CRAFT maintains positive generalization under scenario shift. In-domain fine-tuning on Bench2Drive delivers the largest improvement over the pre-trained policy, while cross-scenario transfer still yields consistent gains. The main limitation appears in Route Completion, where the limited improvement under transfer suggests a more conservative policy in unseen scenarios. Overall, CRAFT generalizes beyond the training split, although the magnitude of improvement remains partly dependent on the interaction distribution used for fine-tuning.

\subsection{Scaling and Training Dynamics}

\Cref{fig:scaling} reveals a clear scaling pattern. CRAFT is the only method whose performance improves consistently with more data. PPO and REINFORCE++, which rely only on real interaction, become less stable at larger scales and show saturating gains, suggesting inefficient learning from sparse closed-loop feedback. GRPO, which depends only on proxy supervision, delivers limited improvement and even deteriorates at larger scales, indicating that proxy bias may become the main bottleneck in later stages of training. These trends demonstrate that efficient scaling requires dense proxy supervision, while sustained gains depend on grounded residual correction.

The results in \Cref{fig:train_stability} further confirm the stability of CRAFT. In terms of KL stability, CRAFT shows larger gains for policies with broader action spaces, such as HiP-AD and SparseDriveV2, while the advantage is smaller for MindDrive, which has a more constrained action space. This pattern suggests that a broader action space improves counterfactual proxy stability. The expected-advantage results support the same conclusion from the grounded side. Compared with GRPO, CRAFT reduces expected-advantage variation across policy architectures. Although grounded residual correction is driven by sparse feedback, dual clipping stabilizes training. Together, these results suggest that the stability of CRAFT arises from both components, with broader action spaces improving proxy stability and dual-clipped grounded residual correction further reducing update variance.

\subsection{Qualitative Scenario Analysis}

\begin{figure}[t]
    \centering
    \includegraphics[width=\textwidth]{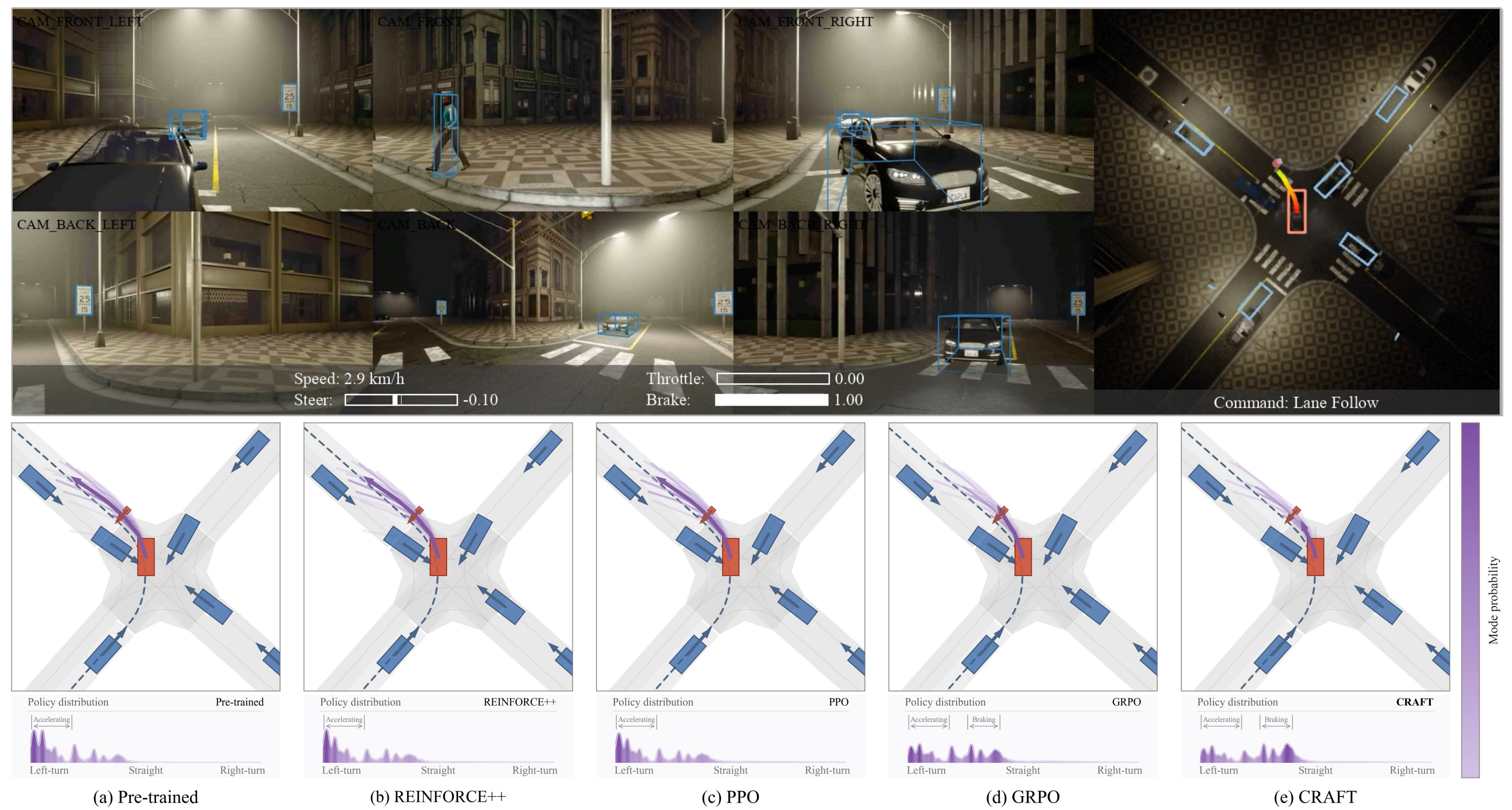}
    \caption{\small \textbf{Policy distribution shift in a pedestrian-crossing scenario.} Top: a closed-loop scene snapshot with SparseDriveV2 as the ego vehicle. Bottom: the policy distributions over the trajectory vocabulary for the pre-trained and fine-tuned policies. CRAFT shifts probability mass toward braking-compatible modes.}
    \label{fig:scenario}
\end{figure}

\Cref{fig:scenario} illustrates how fine-tuning reshapes the policy distribution in a safety-critical state. The pre-trained SparseDriveV2 assigns most probability mass to a left-turn accelerating mode, which remains largely unchanged after PPO and REINFORCE++ fine-tuning. This similarity suggests that sparse closed-loop updates do not efficiently shift the dominant mode away from unsafe behavior. GRPO shifts part of the distribution toward a braking-compatible mode, but this mode remains secondary, indicating that the counterfactual proxy alone gives a biased estimate of the safer closed-loop response. CRAFT induces the clearest shift, making the braking-compatible mode dominant while retaining the original peak of the pre-trained policy at lower probability. These results highlight the advantage of proxy-residual fine-tuning for safety-critical events by shifting probability toward interaction-grounded corrective behavior while preserving the pre-trained policy manifold.

%% file: tab/main_result.tex
\begin{table}[t]
\centering
\caption{\small \textbf{Closed-loop performance on Bench2Drive.}  
${\dagger}$ denotes the vocabulary-based variant of SparseDriveV2.
Best results are highlighted in \textbf{bold}, and second-best results are \underline{underlined}.
}
\label{tab:main_results}
\resizebox{1.0\textwidth}{!}{
\begin{threeparttable}
\tablestyle{2.0pt}{1.1}
\setlength{\tabcolsep}{1mm}{
\newcolumntype{Y}{>{\centering\arraybackslash}m{13mm}}
\begin{tabular}{c!{\color{black}\vrule width 0.3pt}cYY!{\color{black}\vrule width 0.3pt}cccccc}
    \toprule
     \multicolumn{1}{c}{\multirow{2}{*}{\textbf{Driving Policy}}} & \multicolumn{1}{c}{\multirow{2}{*}{\textbf{Method}}} & \multicolumn{2}{c}{\textbf{Closed-loop Metrics}} & \multicolumn{6}{c}{\textbf{Multi-Ability} (\%)} \\
     \cmidrule(r){3-4}
     \cmidrule(r){5-10}
     \multicolumn{1}{c}{} & \multicolumn{1}{c}{} & DS $\uparrow$ & \multicolumn{1}{Y}{SR (\%) $\uparrow$} & Merging $\uparrow$ & Overtaking $\uparrow$ & Emergency Brake $\uparrow$ & Give Way $\uparrow$ & \multicolumn{1}{c}{Traffic Sign $\uparrow$} & Mean $\uparrow$ \\
    \midrule
    \multirow{5}{*}{
    \begin{tabular}[c]{@{}c@{}}
    \textbf{HiP-AD}~\cite{tang2025hipad}\\
    {\scriptsize\textcolor{darkgray}{\textit{(Hierarchical Planning)}}}
    \end{tabular}} 
    & Pre-trained & 73.46 & 53.64 & \textbf{46.30} & 53.30 & 63.30 & \textbf{50.00} & \secondbest{46.30} & 51.80\\
    & REINFORCE++  & 74.82 & 55.00 & 36.20 & \textbf{64.40} & \secondbest{71.70} & \textbf{50.00} & 44.20 & \secondbest{53.30}\\
    & PPO & 73.58 & 53.18 & 41.20 & 55.60 & 68.30 & 30.00 & 43.20 & 47.70\\
    & GRPO & \secondbest{75.27} & \secondbest{55.91} & \secondbest{42.50} & \secondbest{60.00} & 70.00 & \secondbest{40.00} & \textbf{48.40} & 52.20\\
    & \cellcolor{gray!30}CRAFT (Ours) & \cellcolor{gray!30}\textbf{75.79} & \cellcolor{gray!30}\textbf{56.36} & \cellcolor{gray!30}40.00 & \cellcolor{gray!30}\textbf{64.40} & \cellcolor{gray!30}\textbf{73.30} & \cellcolor{gray!30}\textbf{50.00} & \cellcolor{gray!30}\textbf{48.40} & \cellcolor{gray!30}\textbf{55.20}\\
    \midrule
    \multirow{5}{*}{
    \begin{tabular}[c]{@{}c@{}}
    \textbf{MindDrive}~\cite{fu2025minddrive}\\
    {\scriptsize\textcolor{darkgray}{\textit{(Vision Language Action)}}}\end{tabular}}
    & Pre-trained & 62.60 & 40.45 & 26.30 & \secondbest{51.10} & 50.00 & \secondbest{40.00} & 32.60 & 40.00\\
    & REINFORCE++  & 63.68 & 41.36 & \secondbest{32.50} & 46.70 & 50.00 & \textbf{50.00} & 33.70 & 42.60\\
    & PPO & 64.61 & 42.73 & 30.00 & 48.90 & \secondbest{55.00} & \textbf{50.00} & \secondbest{35.80} & 43.90\\
    & GRPO & \secondbest{65.12} & \secondbest{43.18} & 27.50 & \textbf{57.80} & 53.30 & \textbf{50.00} & 34.70 & \secondbest{44.70}\\
    & \cellcolor{gray!30}CRAFT (Ours) & \cellcolor{gray!30}\textbf{67.45} & \cellcolor{gray!30}\textbf{46.36} & \cellcolor{gray!30}\textbf{37.50} & \cellcolor{gray!30}\secondbest{51.10} & \cellcolor{gray!30}\textbf{56.70} & \cellcolor{gray!30}\textbf{50.00} & \cellcolor{gray!30}\textbf{37.90} & \cellcolor{gray!30}\textbf{46.60}\\
    \midrule
    \multirow{5}{*}{
    \begin{tabular}[c]{@{}c@{}}
    \textbf{SparseDriveV2}$^\dagger$~\cite{sun2026sparsedrivev2}\\
    {\scriptsize\textcolor{darkgray}{\textit{(Vocabulary Scoring)}}}
    \end{tabular}} 
    & Pre-trained & 66.43 & 45.45 & \textbf{43.80} & 40.00 & 56.70 & \textbf{50.00} & 32.60 & 44.60\\
    & REINFORCE++  & 66.60 & 46.36 & \secondbest{41.20} & \textbf{53.30} & 55.00 & \textbf{50.00} & 26.30 & 45.20\\
    & PPO & 67.44 & \secondbest{49.09} & \textbf{43.80} & \textbf{53.30} & \secondbest{61.70} & \textbf{50.00} & 33.70 & \secondbest{48.50}\\
    & GRPO & \secondbest{68.40} & \secondbest{48.18} & \secondbest{41.20} & 48.90 & \secondbest{61.70} & \textbf{50.00} & \secondbest{36.80} & 47.70\\
    & \cellcolor{gray!30}CRAFT (Ours) & \cellcolor{gray!30}\textbf{71.10} & \cellcolor{gray!30}\textbf{50.45} & \cellcolor{gray!30}40.00 & \cellcolor{gray!30}\textbf{53.30} & \cellcolor{gray!30}\textbf{68.30} & \cellcolor{gray!30}\textbf{50.00} & \cellcolor{gray!30}\textbf{37.90} & \cellcolor{gray!30}\textbf{49.90}\\
    \bottomrule
\end{tabular}
}
\end{threeparttable}
}
\end{table}

%% file: tab/ablation.tex
\begin{table*}[t]
\centering

\begin{minipage}[t]{0.42\textwidth}
\centering

\renewcommand\arraystretch{1.05}
\setlength{\tabcolsep}{4pt}
\small
\captionof{table}{\textbf{Ablation study of CRAFT.} \texttt{SD}: self-distillation; \texttt{CP}: counterfactual proxy; \texttt{GR}: grounded residual.}
\label{tab:craft_ablation}
\resizebox{\linewidth}{!}{
\begin{tabular}{@{}c c c c c c c c@{}}
\toprule
\multirow{2}{*}{ID} & \multirow{2}{*}{SD} & \multirow{2}{*}{CP} & \multirow{2}{*}{GR} & \multicolumn{4}{c}{Closed-loop Metrics}\\
\cmidrule(lr){5-8}
& & & & DS $\uparrow$& RC $\uparrow$& IS $\uparrow$& SR $\uparrow$\\
\midrule
\rowcolor{gray!0}
1 & \xmark & \xmark & \xmark & 66.43 & 77.53 & 0.80 & 45.45\\
\rowcolor{gray!10}
2 & \cmark & \xmark & \xmark & 66.49 & 76.65 & 0.81 & 46.36\\
\rowcolor{gray!20}
3 & \cmark & \cmark & \xmark & 68.40 & 78.28 & 0.82 & 48.18\\
\rowcolor{gray!35}
4 & \cmark & \cmark & \cmark & \textbf{71.10} & \textbf{80.01} & \textbf{0.84} & \textbf{50.45}\\
\bottomrule
\end{tabular}
}
\end{minipage}
\hfill
\begin{minipage}[t]{0.56\textwidth}
\centering
\renewcommand\arraystretch{1.1}
\setlength{\tabcolsep}{2.8pt}
\small
\captionof{table}{\textbf{Generalization across datasets.} 
Gains over the pre-trained policy under each train-test split. 
${\ddagger}$: Town12 subset of LEAD scenarios.}
\label{tab:transfer}
\resizebox{\linewidth}{!}{
\begin{tabular}{@{}c c c c c c@{}}
\toprule
\multirow{2}{*}{Train} & \multirow{2}{*}{Test} & \multicolumn{4}{c}{Closed-loop Metrics}\\
\cmidrule(lr){3-6}
& & $\Delta$ DS $\uparrow$& $\Delta$ RC $\uparrow$& $\Delta$ IS $\uparrow$& $\Delta$ SR $\uparrow$\\
\midrule
Bench2Drive & Bench2Drive & {\textbf{+4.67}} & {\textbf{+2.48}} & {\textbf{+0.04}} & {\textbf{+5.00}}\\
\rowcolor{gray!30}
Bench2Drive & Longest6 v2~\cite{jaeger2025carl} & {+2.18} & {+1.12} & {+0.06} & {+0.00} \\
LEAD$^{\ddagger}$ & LEAD$^{\ddagger}$~\cite{Nguyen2026lead} & {\textbf{+2.14}} & {\textbf{+2.60}} & {\textbf{+0.01}} & {\textbf{+1.89}}\\
\rowcolor{gray!30}
LEAD$^{\ddagger}$ & Bench2Drive & {+0.58} & {-0.34} & {+0.01} & {+0.91}\\

\bottomrule
\end{tabular}
}
\end{minipage}
\end{table*}

%% file: tex/5_conclusions.tex
\section{Conclusions}
\label{sec:conclusion}

This paper presents CRAFT, a proxy-residual framework for closed-loop post-training of supervised driving policies. CRAFT combines a dense counterfactual proxy on on-policy visited states, grounded residual correction from executed rollouts, and self-distillation to preserve reliable pre-trained behavior during fine-tuning. We support this design with an exact gradient decomposition on the real closed-loop state distribution, together with variance, bias, and stability analyses. On Bench2Drive, CRAFT achieves the strongest closed-loop gains across hierarchical-planning, vision-language-action, and vocabulary-scoring architectures, while ablation and scaling results confirm the complementary roles of counterfactual proxy learning, grounded residual correction, and self-distillation.


%% file: tex/6_appendix.tex
\phantomsection
\section*{Appendix}

\section{Proofs}
\label{sec:theory}

This appendix provides the proofs for the formal statements in \Cref{sec:method}. Unless noted otherwise, expectations are taken over $s\sim d_{\mathrm{real}}^{\pi_\theta}$ and candidate trajectory $\tau\sim\pi_\theta(\cdot\mid s)$.

\subsection{Exact Decomposition}
\label{app:proof_exact_decomposition}

\begin{proof}
By definition of the residual mismatch,
\begin{equation}
A^{\mathrm{real}}_{\pi_\theta}(s,\tau)
=
\Phi_{\pi_\theta}(s,\tau)+\Delta_{\pi_\theta}(s,\tau).
\end{equation}
Substituting this identity into the real policy gradient in \Cref{eq:real_pg} gives
\begin{align}
\nabla_\theta J_{\mathrm{real}}(\pi_\theta)
&=
\mathbb{E}
\Big[
\nabla_\theta \log \pi_\theta(\tau\mid s)\,
\big(\Phi_{\pi_\theta}(s,\tau)+\Delta_{\pi_\theta}(s,\tau)\big)
\Big] \\
&=
\mathbb{E}
\Big[
\nabla_\theta \log \pi_\theta(\tau\mid s)\,
\Phi_{\pi_\theta}(s,\tau)
\Big]
+
\mathbb{E}
\Big[
\nabla_\theta \log \pi_\theta(\tau\mid s)\,
\Delta_{\pi_\theta}(s,\tau)
\Big],
\end{align}
which proves the claim.
\end{proof}

\subsection{KL-Proximal Policy Improvement}
\label{app:proof_kl_proximal}

\begin{proposition}[KL-Proximal Policy Improvement]
\label{prop:kl_proximal}
For a finite or countable discrete trajectory space, fix state $s$ and a strictly positive teacher distribution $q(\tau)=\pi_T(\tau\mid s)$. For any fixed bounded score $U(s,\tau)$ and temperature $\eta>0$, the non-parametric solution to
\begin{equation}
\max_{p\in\Delta(\mathcal{T})}\;
\mathbb{E}_{\tau\sim p}[U(s,\tau)]-\eta D_{\mathrm{KL}}(p\|q)
\end{equation}
is
\begin{equation}
p_\eta^\star(\tau\mid s)
=
\frac{
q(\tau)\exp(U(s,\tau)/\eta)
}{
\sum_{\tau'\in\mathcal{T}}q(\tau')\exp(U(s,\tau')/\eta)
}.
\label{eq:kl_proximal_solution}
\end{equation}
\end{proposition}

\begin{proof}
Fix $s$ and write $U(\tau)=U(s,\tau)$. The Lagrangian for the optimization problem in Proposition~\ref{prop:kl_proximal}, with multiplier $\lambda$ for the normalization constraint $\sum_{\tau}p(\tau)=1$, is
\begin{equation}
\mathcal{L}(p,\lambda)
=
\sum_{\tau}p(\tau)U(\tau)
-
\eta\sum_{\tau}p(\tau)\log\frac{p(\tau)}{q(\tau)}
+
\lambda\Big(\sum_\tau p(\tau)-1\Big).
\end{equation}
For any candidate trajectory with $p(\tau)>0$, stationarity gives
\begin{equation}
U(\tau)-\eta\Big(\log\frac{p(\tau)}{q(\tau)}+1\Big)+\lambda=0,
\end{equation}
and therefore
\begin{equation}
p(\tau)
=
q(\tau)\exp\!\left(\frac{U(\tau)+\lambda-\eta}{\eta}\right).
\end{equation}
The normalizing constraint determines the common factor,
\begin{equation}
\exp\!\left(\frac{\lambda-\eta}{\eta}\right)
=
\left(\sum_{\tau'}q(\tau')\exp(U(\tau')/\eta)\right)^{-1}.
\end{equation}
Substituting this factor yields \Cref{eq:kl_proximal_solution}. Since the objective is strictly concave in $p$ when $\eta>0$ and $q$ has full support, this stationary point is the unique maximizer.
\end{proof}

\subsection{Variance Reduction}
\label{app:proof_variance_reduction}

\begin{proof}
Fix $s$ and suppress the dependence on $(s,\tau)$ for readability. Since $R_\alpha=Y-\alpha C$,
\begin{align}
\Var[R_\alpha]
&=
\Var[Y-\alpha C] \\
&=
\Var[Y]+\Var[\alpha C]-2\Cov[Y,\alpha C] \\
&=
\Var[Y]+\alpha^2\Var[C]-2\alpha\Cov[Y,C],
\end{align}
which proves the variance identity. Differentiating with respect to $\alpha$ gives
\begin{equation}
\frac{\partial}{\partial \alpha}\Var[R_\alpha]
=
2\alpha\Var[C]-2\Cov[Y,C].
\end{equation}
Because $\Var[C]>0$, the unique stationary point is $\alpha^\star=\Cov[Y,C]/\Var[C]$, and the second derivative is $2\Var[C]>0$, so $\alpha^\star$ is the global minimizer.

For $\alpha=1$,
\begin{equation}
\Var[Y-C]-\Var[Y]=\Var[C]-2\Cov[Y,C].
\end{equation}
Thus $\Cov[Y,C]>\frac12\Var[C]$ implies $\Var[Y-C]<\Var[Y]$. If the condition holds for $d_{\mathrm{real}}^{\pi_\theta}$-almost every $s$, taking expectation over $s$ preserves the strict inequality for the expected conditional variances.
\end{proof}

\subsection{Dual Clipping}
\label{app:proof_dualclip}

\begin{proposition}[Bounded Dual-Clipped Surrogate]
\label{prop:dualclip}
Suppose $|\hat A^{\mathrm{gr}}_t|\le A_{\max}$ for all $t$ after clipping, $0<\epsilon_{\mathrm{clip}}<1$, $c>1$, and $\rho_t(\theta)>0$. With $\ell_t(\theta)$ defined in \Cref{eq:lt}, define
\begin{equation}
u_t(\theta)
=
\max(\ell_t(\theta),c\hat A^{\mathrm{gr}}_t)\mathbf{1}_{\hat A^{\mathrm{gr}}_t<0}
+
\ell_t(\theta)\mathbf{1}_{\hat A^{\mathrm{gr}}_t\ge 0}.
\end{equation}
Then, for every $t$,
\begin{equation}
-cA_{\max}\le u_t(\theta)\le (1+\epsilon_{\mathrm{clip}})A_{\max}.
\end{equation}
\end{proposition}

\begin{proof}
If $\hat A^{\mathrm{gr}}_t\ge 0$, then $u_t(\theta)=\ell_t(\theta)$. Since $\mathrm{clip}(\rho_t(\theta),1-\epsilon_{\mathrm{clip}},1+\epsilon_{\mathrm{clip}})\le 1+\epsilon_{\mathrm{clip}}$, the clipped surrogate satisfies
\begin{equation}
\ell_t(\theta)\le (1+\epsilon_{\mathrm{clip}})\hat A^{\mathrm{gr}}_t\le (1+\epsilon_{\mathrm{clip}})A_{\max}.
\end{equation}
Moreover $\ell_t(\theta)\ge 0$ in this case, and therefore $-cA_{\max}\le u_t(\theta)\le (1+\epsilon_{\mathrm{clip}})A_{\max}$.

If $\hat A^{\mathrm{gr}}_t<0$, then $u_t(\theta)=\max(\ell_t(\theta),c\hat A^{\mathrm{gr}}_t)$. Since $c>1$ and $\hat A^{\mathrm{gr}}_t\ge -A_{\max}$, we have $c\hat A^{\mathrm{gr}}_t\ge -cA_{\max}$, so
\begin{equation}
u_t(\theta)\ge c\hat A^{\mathrm{gr}}_t\ge -cA_{\max}.
\end{equation}
Both $\ell_t(\theta)$ and $c\hat A^{\mathrm{gr}}_t$ are non-positive when $\hat A^{\mathrm{gr}}_t<0$, hence $u_t(\theta)\le 0\le (1+\epsilon_{\mathrm{clip}})A_{\max}$. Combining the two cases proves the bound.
\end{proof}

\subsection{Residual Approximation Bias}
\label{app:proof_bias_compensation}

\begin{proof}
From \Cref{eq:exact_decomposition}, the real gradient can be written as
\begin{equation}
g(\theta)
=
\mathbb{E}
\Big[
\nabla_\theta\log\pi_\theta(\tau\mid s)\,
\big(\Phi_{\pi_\theta}(s,\tau)+\Delta_{\pi_\theta}(s,\tau)\big)
\Big].
\end{equation}
Subtracting this expression from the surrogate gradient in \Cref{eq:ghat} yields
\begin{equation}
\widehat g(\theta)-g(\theta)
=
\mathbb{E}
\Big[
\nabla_\theta\log\pi_\theta(\tau\mid s)\,
\big(\widehat{\Delta}_{\pi_\theta}(s,\tau)-\Delta_{\pi_\theta}(s,\tau)\big)
\Big].
\end{equation}
Let $X=\nabla_\theta\log\pi_\theta(\tau\mid s)$ and $e=\widehat{\Delta}_{\pi_\theta}(s,\tau)-\Delta_{\pi_\theta}(s,\tau)$. Jensen's inequality gives
\begin{equation}
\|\widehat g(\theta)-g(\theta)\|_2
=
\|\mathbb{E}[Xe]\|_2
\le
\mathbb{E}[\|X\|_2|e|].
\end{equation}
Applying Cauchy--Schwarz to the scalar random variables $\|X\|_2$ and $|e|$ gives
\begin{equation}
\mathbb{E}[\|X\|_2|e|]
\le
\big(\mathbb{E}\|X\|_2^2\big)^{1/2}
\big(\mathbb{E}|e|^2\big)^{1/2},
\end{equation}
which proves \Cref{eq:bias_bound_residual}. If $\|X\|_2\le G_{\nabla}$ almost surely, then $\big(\mathbb{E}\|X\|_2^2\big)^{1/2}\le G_{\nabla}$, yielding the bounded-gradient version.
\end{proof}

\section{Bench2Drive Protocol}
\label{app:bench2drive_patch}

\subsection{Motivation for Adjustments}
\label{app:bench2drive_patch_motivation}

Closed-loop training imposes fundamentally different requirements on a benchmark than closed-loop evaluation. Beyond accurate assessment, training requires clear and informative failure signals, as well as visually and physically consistent environment dynamics across policy updates to avoid non-stationarity, together with high rollout efficiency. In particular, rollouts should terminate immediately upon meaningful failure events to provide concise supervision, and unnecessary criteria overhead should be minimized, as rollout time directly dictates training efficiency.

We therefore adopt a patched protocol, Bench2Drive$^{*}$, which preserves the route set, ego-policy interface, and core metrics, while mitigating scenario artifacts that introduce non-policy noise into both training and evaluation.

\subsection{Protocol-Level Adjustments}
\label{app:bench2drive_patch_fixes}

We apply six local adjustments to the official protocol.

\begin{itemize}[leftmargin=1.2em,itemsep=0.25em,topsep=0.25em]
\item \textbf{Actor cleanup}. The official scenario manager can remove background or scenario-related actors at runtime, causing vehicles to disappear from the visual scene. We now clean up only background vehicles that block normal ego progress. Other nonblocking vehicles are kept in autopilot mode, and actors that would be removed immediately after spawning are skipped before they are spawned.

\item \textbf{Timeout handling}. During both data collection and evaluation, we lower the upper bounds for scenario and route timeouts. This prevents a deadlocked route from consuming long simulator time after the policy has already produced the relevant failure signal.

\item \textbf{BicycleFlow spawning}. We fix the BicycleFlow scenario so that cyclists are spawned only within the range specified by the scenario definition. The resulting flow follows normal behavior and avoids abnormal cyclist motion or collisions that are unrelated to the ego policy.

\item \textbf{Route 3749 actor flow}. For route id 3749, we set the actor-flow speed to 12 m/s. This avoids an overly aggressive flow for the turning geometry, where the original setting could produce an invalid scene before the ego policy makes a meaningful decision.

\item \textbf{Collision termination}. In both data collection and evaluation, a collision terminates the scene immediately. The original protocol may continue simulation after a collision; in BicycleFlow, for example, a cyclist can keep moving and generate repeated contacts from the same physical event. Early termination makes the failure state unambiguous and reduces noise in Driving Score. It may slightly lower Driving Score by removing post-collision driving segments, but it does not change Success Rate. We also ignore collision events with relative speed below 0.1 m/s to avoid duplicate contact records.

\item \textbf{Town07 stop sign}. We ignore the invisible stop sign located at $(0,0,0)$ in Town07. This removes spurious stop-sign violations caused by a map artifact rather than by policy behavior.
\end{itemize}

\input{tab/benchmark_robust.tex}

\subsection{Sanity Check Against Official Bench2Drive}
\label{app:bench2drive_patch_effect}

The protocol adjustments aim to produce cleaner RL rollouts without altering the benchmark's qualitative semantics. We therefore compare the official protocol with Bench2Drive$^{*}$ across five representative policies.

\begin{itemize}[leftmargin=1.2em,itemsep=0.2em,topsep=0.2em]
\item PDM-Lite~\cite{sima2024drivelm} is a rule-based CARLA planner. It provides a nonlearned reference point whose behavior is driven by hand-designed planning logic rather than end-to-end perception and control.

\item PlanT v2~\cite{gerstenecker2025plant2} is a privileged planner with access to structured scene information. We include it to test whether the patched protocol remains consistent for high-performing planners that operate outside the sensor-only setting.

\item UniAD~\cite{hu2023_uniad} is a planning-oriented end-to-end driving system. We use it as a representative learned policy with integrated perception, prediction, and planning modules.

\item VAD~\cite{jiang2023vad} is an end-to-end method based on vectorized scene representations. It offers a learned baseline with a different intermediate representation from UniAD while staying in the same sensor-driven evaluation regime.

\item SparseDrive~\cite{sun2025sparsedrive} is an end-to-end method built around sparse scene representations. It is included because its closed-loop scores are close to the regime studied in this paper and because CRAFT is evaluated on SparseDrive-style policies.
\end{itemize}

Together, these policies cover rule-based, privileged, and end-to-end settings, allowing us to check whether the patched protocol changes efficiency and rollout cleanliness without altering the broad ordering of representative methods.

\Cref{tab:benchmark_robust} shows the expected pattern. The adjusted protocol is slightly stricter for strong planners, mainly because terminal events are handled more directly and avoidable post-failure rollout time is removed. For end-to-end methods, the ranking remains broadly stable. UniAD remains ahead of VAD and SparseDrive. Moreover, Bench2Drive$^{*}$ produces substantially lower standard deviations than the official protocol, suggesting a more stable evaluation of closed-loop performance. This lower variance supports single-seed evaluation in the main experiments as a practical and cost-effective protocol for comparing relative method improvements under identical evaluation conditions. In practice, the shorter timeout and cleaner termination rules yield over 2$\times$ higher closed-loop training throughput and over 3$\times$ higher evaluation throughput. These efficiency gains are critical for reinforcement fine-tuning, where closed-loop rollout time is the primary bottleneck.

\section{Algorithm Framework}
\label{app:algorithm_framework}

Algorithm~\ref{alg:craft} gives the training procedure that instantiates the objective in \Cref{eq:overall_loss}. The loop follows the separation used throughout the method section. At each closed-loop state, the policy scores a finite set of candidate trajectories, and only the selected trajectory is passed to the controller to produce the action executed in the real environment. The counterfactual world later evaluates the stored candidate set from the same visited state, where counterfactual returns provide group-normalized counterfactual advantages without executing alternative candidates in CARLA. In parallel, the corrective reward is computed from the executed controller action and the resulting next state, then discounted, scaled, and clipped into the corrective advantage. The update combines the counterfactual advantage and corrective advantage with EMA teacher regularization, aligning the algorithm with the proxy-residual formulation in \Cref{sec:method} and the implementation details reported below.

\begin{algorithm}[ht]
\caption{Counterfactual-to-Interactive Reinforcement Fine-Tuning (CRAFT)}
\label{alg:craft}
\begin{algorithmic}[1]
\Require Pre-trained driving policy $\pi_{\theta_{\mathrm{pre}}}$, controller $\kappa$, closed-loop environment $\mathcal{M}_{\mathrm{real}}$, counterfactual world $\mathcal{E}_{\mathrm{cp}}$, rollout buffer size $B$, candidate group size $G$, weights $\lambda_{\mathrm{cp}},\lambda_{\mathrm{gr}},\beta_r,\beta_f$, EMA momentum $m$
\Ensure Fine-tuned policy $\pi_\theta$
\State Initialize $\theta\leftarrow\theta_{\mathrm{pre}}$ and $\theta_T\leftarrow\theta_{\mathrm{pre}}$
\For{training iteration $k=1,2,\ldots$}
    \State Set $\theta_{\mathrm{old}}\leftarrow\theta$ and clear $\mathcal{B}$
    \While{$|\mathcal{B}|<B$}
        \State Observe $s_t$ and candidates $\{\tau_{t,g}\}_{g=1}^G$
        \State Sample \(\tau_t \sim \pi_{\theta_{\mathrm{old}}}(\cdot\mid s_t)\) and execute $a_t=\kappa(s_t,\tau_t)$ \Comment{closed-loop execution}
        \State Observe $s_{t+1}$ and $r_t^{\mathrm{gr}}=r^{\mathrm{gr}}(s_t,a_t,s_{t+1})$
        \State Store $(s_t,\{\tau_{t,g}\}_{g=1}^G,\tau_t,\log\pi_{\theta_{\mathrm{old}}}(\tau_t\mid s_t),r_t^{\mathrm{gr}})$ in $\mathcal{B}$
    \EndWhile
    \For{each visited state $s_i$ in $\mathcal{B}$}
        \State Run $\mathcal{E}_{\mathrm{cp}}$ on $\{\tau_{i,g}\}_{g=1}^G$ to obtain counterfactual returns $\tilde R_{i,g}$ \Comment{counterfactual reward}
        \State Mask invalid candidates and normalize valid returns to $\tilde A_{i,g}$ as in \Cref{eq:group_adv}
    \EndFor
    \State Compute $\mathcal{L}_{\mathrm{cp}}$ using \Cref{eq:lcp}
    \State Discount corrective rewards with $\gamma_c$, then scale and clip to $\hat A_t^{\mathrm{gr}}$ \Comment{corrective advantage}
    \State Form $\rho_t(\theta)=\pi_\theta(\tau_t\mid s_t)/\pi_{\theta_{\mathrm{old}}}(\tau_t\mid s_t)$ and compute $\mathcal{L}_{\mathrm{gr}}$ using \Cref{eq:lt,eq:lgr}
    \State Compute $\mathcal{L}_{\mathrm{dist}}$ and $\mathcal{L}_{\mathrm{KL}}$ using \Cref{eq:kls} \Comment{on-policy self-distillation}
    \State Update trainable policy parameters by minimizing $\mathcal{L}_{\mathrm{all}}$ in \Cref{eq:overall_loss}
    \State Update $\theta_T\leftarrow m\theta_T+(1-m)\theta$
\EndFor
\State \Return $\pi_\theta$
\end{algorithmic}
\end{algorithm}

\section{Implementation Details}
\label{app:training_details}

This section records the implementation choices used in the experiments in \Cref{sec:exp}. The main paper keeps the method description abstract; here we specify which policy parameters are updated, how candidate futures are simulated, and how the reward and optimization terms are instantiated.

\subsection{Compute Resources}
\label{app:compute_resources}

Our experiments are conducted on a single server equipped with eight NVIDIA RTX 4090D GPUs. CRAFT uses an asynchronous collect-and-train pipeline, where rollout workers interact with CARLA and write synchronized rollout snapshots to a shared buffer. The trainer then waits until enough rollout files are ready and aggregates them into a rollout batch, preprocesses the counterfactual and corrective signals, and performs one Lightning\footnote{\url{https://github.com/Lightning-AI/pytorch-lightning}} training iteration before releasing the collectors to continue with the updated policy version. Each driving policy uses four GPUs for closed-loop rollout collection and four GPUs for policy training; under this configuration, one complete fine-tuning run takes 6--10 hours, depending on the policy's inference efficiency and the trainable modules. This design overlaps closed-loop data collection with policy optimization, which reduces idle time in the main training loop.

\subsection{Trainable Modules of Driving Policies}
\label{app:trainable_modules}

We fine-tune only the modules that score candidate trajectories. All perception, map construction, detection, motion prediction, and trajectory-generation modules remain frozen. This restriction keeps sparse online rewards close to the policy decision boundary. The update can change which trajectory candidate is preferred, but it does not rewrite the visual or geometric driving stack. \Cref{fig:ego_policy_paradigms} illustrates the operational paradigms of these driving policies.

\begin{itemize}[leftmargin=1.2em,itemsep=0.2em,topsep=0.2em]
\item \textbf{HiP-AD~\cite{tang2025hipad}.} The ego decision is represented as joint path-speed candidates with 48 ego modes. CRAFT and the non-PPO baselines update only the final classification branch in the last plan-refinement layer. When PPO is used, a lightweight critic is trained in addition to this branch.

\item \textbf{MindDrive~\cite{fu2025minddrive}.} The high-level speed decision is chosen by a VLM decision expert before the policy generates the low-level path. Fine-tuning updates only the LoRA decision expert used for this high-level decision. The visual encoder, map and detection heads, language backbone, and action expert stay fixed. Under PPO, the critic is the only additional trainable component.

\item \textbf{SparseDriveV2$^{\dagger}$~\cite{sun2026sparsedrivev2}.} The policy scores a 512-mode trajectory vocabulary. Fine-tuning updates only the final trajectory classification branch that assigns probabilities to vocabulary candidates. The trajectory vocabulary, temporal queue, perception backbone, map head, detection head, and trajectory decoder remain frozen. PPO adds the critic, while all other fine-tuning methods update only the classification branch.
\end{itemize}

\begin{figure}[t]
    \centering
    \includegraphics[width=\textwidth]{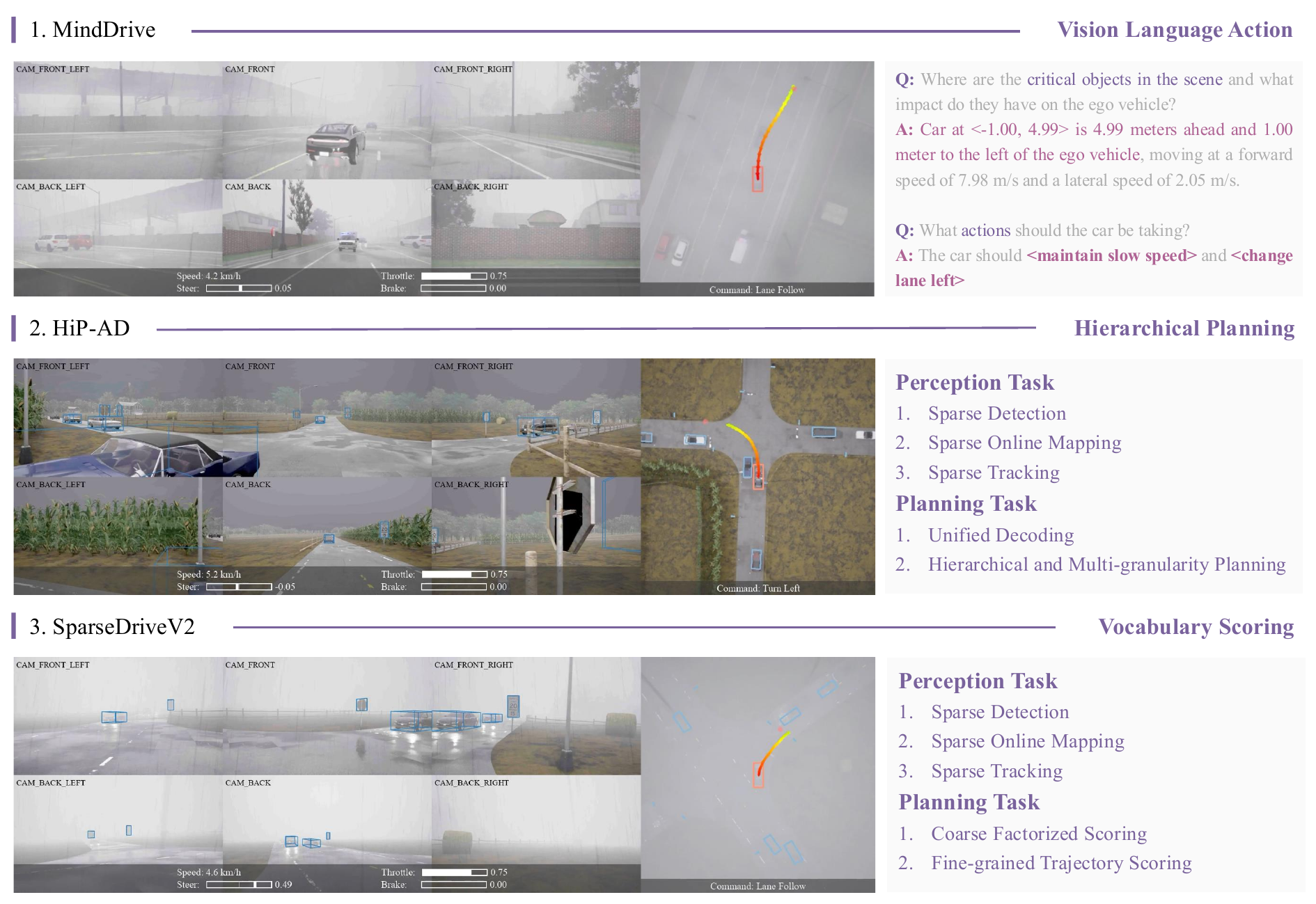}
    \caption{\small \textbf{Operational paradigms of driving policies.}}
    \label{fig:ego_policy_paradigms}
\end{figure}

\input{tab/implementation_hparams}

\subsection{Counterfactual World}
\label{app:counterfactual_engine}

The counterfactual world evaluates the policy candidate set at states visited by real closed-loop rollouts. For each observation, the Bench2Drive agent stores an auxiliary record containing candidate trajectories $\tau_{1:G}$ in the local ego frame, candidate logits, the current ego state, nearby-agent states and actions, route geometry, scenario flags, and nearby map polygons. Ragged candidate counts, route lengths, and agent counts are padded into dense tensors with validity masks before simulation.

By default, each candidate is simulated with a batched trajectory tracker. The tracker prepends the current ego pose to each candidate trajectory, transforms the local trajectory into the global CARLA frame, and flattens the batch and candidate dimensions into independent tracking problems. At each virtual step, a PID controller tracks a short reference segment, and a kinematic bicycle model advances the ego state. Nearby agents use a decay rollout model in which their current actions are held for a short decision window. The model then reduces throttle and ramps up braking, with a shorter ramp when an agent is on a lane connector. Collision checks use rectangular vehicle footprints and batched separating-axis tests.

The counterfactual return for candidate $g$ at real state $s_i$ is
\begin{equation}
\tilde R_{i,g}
=
\sum_{t=0}^{H-1}\gamma^t r^{\mathrm{cp}}_{i,g,t}
-\lambda^{\mathrm{cp}}_{\mathrm{red}}I^{\mathrm{red}}_{i,g}
-\lambda^{\mathrm{cp}}_{\mathrm{stop}}I^{\mathrm{stop}}_{i,g}.
\label{eq:app_cp_return}
\end{equation}
The rollout horizon $H$, simulator step $\Delta t$, and discount $\gamma$ are listed in \Cref{tab:implementation_hparams}; in this setting, the simulated rollout contains the current state plus $H$ future states. Candidate returns are normalized within each valid group as follows:
\begin{equation}
\mu_i=\frac{1}{|G_i|}\sum_{g\in G_i}\tilde R_{i,g},
\qquad
\sigma_i=
\max\!\left(
\sqrt{\frac{1}{|G_i|}\sum_{g\in G_i}(\tilde R_{i,g}-\mu_i)^2},
\sigma_{\min}
\right),
\end{equation}
\begin{equation}
\tilde A_{i,g}=\frac{\tilde R_{i,g}-\mu_i}{\sigma_i}.
\label{eq:app_group_adv}
\end{equation}
Padded candidates are masked out and receive zero advantage. This is the group advantage used by GRPO and by the counterfactual proxy of CRAFT.

\subsection{Reward Design}
\label{app:reward_design}

The implementation uses two rewards. The counterfactual reward is a dense progress-and-infraction signal used to score candidate futures in the counterfactual world; PPO and REINFORCE++ use the closed-loop instance of the same dense design as their rollout reward. The corrective reward is a sparse safety signal used to instantiate grounded residual correction in CRAFT. The two rewards implement the counterfactual proxy and grounded residual correction, respectively. Both designs are inspired by the progress, route-adherence, and safety terms of CaRL~\cite{jaeger2025carl}.

\input{tab/reward_hparams}

\paragraph{Counterfactual reward.}
\label{app:counterfactual_reward}

The counterfactual reward scores efficiency, route adherence, traffic-rule compliance, and collision safety for each simulated candidate. Progress is clipped and normalized as
\begin{equation}
\bar{\Delta p}_t=\mathrm{clip}(\Delta p_t,p_{\min},p_{\max}),
\qquad
\Delta\hat p_t=\frac{\bar{\Delta p}_t}{p_{\max}}.
\end{equation}
The route-efficiency multiplier penalizes route, lane-center, and heading deviations as
\begin{equation}
\eta_t=
\max\!\left(
\exp(-w_g d^g_t)\exp(-w_c d^c_t)\exp(-w_h d^h_t),
\eta_{\min}
\right).
\end{equation}
The recovery term rewards reductions in existing deviations through
\begin{equation}
c_t=
\mathbf{1}[d^g_t>\delta_g^{\mathrm{rec}}]k_g(-\bar{\Delta d}^g_t)
+\mathbf{1}[d^c_t>\delta_c^{\mathrm{rec}}]k_c(-\bar{\Delta d}^c_t)
+\mathbf{1}[d^h_t>\delta_c^{\mathrm{rec}}]k_h(-\bar{\Delta d}^h_t),
\end{equation}
\begin{equation}
r^{\mathrm{rec}}_t
=
\mathrm{clip}(c_t,-c_{\mathrm{clip}},c_{\mathrm{clip}})
\left(a_{\mathrm{rec}}+b_{\mathrm{rec}}\Delta\hat p_t\right),
\end{equation}
where each deviation delta is clipped to $[-d_{\mathrm{clip}},d_{\mathrm{clip}}]$. \Cref{tab:reward_hparams} reports separate route and centerline recovery thresholds; the heading recovery term intentionally reuses the centerline threshold because both deviations are activated at the same local-lane recovery boundary. The collision penalty is scaled by
\begin{equation}
m^{\mathrm{coll}}_t
=
1+\alpha_v\,\mathrm{clip}\!\left(\frac{v_t}{v_{\mathrm{ref}}},\nu_{\min},\nu_{\max}\right).
\end{equation}
The counterfactual reward parameters are summarized in \Cref{tab:reward_hparams}. The per-step counterfactual reward is
\begin{align}
r^{\mathrm{cp}}_t
&=
w_{\mathrm{prog}}\Delta\hat p_t\eta_t
+ r^{\mathrm{rec}}_t
- \lambda^{\mathrm{cp}}_{\mathrm{offroad}} I_{\mathrm{offroad}}
- \lambda^{\mathrm{cp}}_{\mathrm{opp}} I_{\mathrm{opposite}}
- \lambda^{\mathrm{cp}}_{\mathrm{offroute}} I_{\mathrm{offroute}} \notag\\
&\quad
- \lambda^{\mathrm{cp}}_{\mathrm{emg}} I_{\mathrm{emergency}}
- \lambda^{\mathrm{cp}}_{\mathrm{coll}} m^{\mathrm{coll}}_t I_{\mathrm{collision}} .
\label{eq:app_counterfactual_reward}
\end{align}
In the counterfactual world, red-light and stop-sign violations are trajectory-level flags, giving the return-level penalties in \Cref{eq:app_cp_return}. The implementation applies these penalties at the first reward index of the simulated rollout. Collision cost is applied only at the first collision event for each candidate, so high-speed first impacts receive larger penalties.

\paragraph{Corrective reward.}
\label{app:corrective_reward}

The corrective reward is intentionally sparse and safety centered. It is computed after the selected trajectory is converted by the controller into the action actually executed in CARLA, not from alternative virtual candidates:
\begin{align}
r^{\mathrm{gr}}_t
&=
-\lambda^{\mathrm{gr}}_{\mathrm{offroad}} I_{\mathrm{offroad}}
-\lambda^{\mathrm{gr}}_{\mathrm{emg}} I_{\mathrm{emergency}}
-\lambda^{\mathrm{gr}}_{\mathrm{offroute}} I_{\mathrm{offroute}}
-\lambda^{\mathrm{gr}}_{\mathrm{red}} I_{\mathrm{red}}
\notag\\
&\quad
-\lambda^{\mathrm{gr}}_{\mathrm{stop}} I_{\mathrm{stop}}
-\lambda^{\mathrm{gr}}_{\mathrm{coll}} I_{\mathrm{collision}} .
\label{eq:app_corrective_reward}
\end{align}
The traffic-light and stop-sign terms distinguish uncontrolled zones, stop-required zones, and go-required zones. In a stop-required zone, moving faster than $v_{\mathrm{stop}}$ is penalized; in a go-required zone, the wrapper checks whether the path ahead is clear before penalizing motion below $v_{\mathrm{go}}$.

\subsection{Fine-tuning Baselines}
\label{app:objective_details}

The fine-tuners differ in how they consume closed-loop and counterfactual signals. PPO and REINFORCE++ use the dense progress-and-infraction reward on executed rollouts, GRPO uses only counterfactual advantages, and CRAFT combines the counterfactual proxy with the corrective reward.

\paragraph{PPO baseline.}
PPO is the only compared fine-tuner that trains a critic. The datamodule computes generalized advantage estimates with discount $\gamma$ and trace parameter $\lambda_{\mathrm{GAE}}$:
\begin{equation}
\delta_t=r_t+\gamma u^{\mathrm{term}}_tV(s_{t+1})-V(s_t),
\qquad
A_t=\delta_t+\gamma\lambda_{\mathrm{GAE}} u^{\mathrm{done}}_tA_{t+1}.
\end{equation}
Here $u^{\mathrm{term}}_t$ is the continuation mask for value bootstrapping after terminal states, and $u^{\mathrm{done}}_t$ is the continuation mask for recursive return or advantage computation after episode boundaries.
The policy loss uses the standard clipped ratio with range $\epsilon_{\mathrm{clip}}$, and the value loss uses a clipped SmoothL1 target with value-clip range $\epsilon_v$. The critic parameter group uses the actor learning rate multiplied by $\kappa_v$.

\paragraph{REINFORCE++ baseline.}
This value-free baseline removes the critic and uses normalized Monte Carlo returns from the executed rollout,
\begin{equation}
G_t=r_t+\gamma u^{\mathrm{done}}_tG_{t+1},
\qquad
A_t=\frac{G_t-\mu_G}{\sigma_G+\varepsilon_{\mathrm{norm}}}.
\end{equation}
It then uses the same selected-trajectory clipped ratio as PPO, with the shared clip range $\epsilon_{\mathrm{clip}}$.

\paragraph{GRPO baseline.}
This group-relative baseline discards the closed-loop rollout reward and optimizes the normalized counterfactual advantages in \Cref{eq:app_group_adv}. For valid candidates, its objective is
\begin{equation}
\mathcal{L}_{\mathrm{GRPO}}
=
-\mathbb{E}_{i}
\left[
\sum_{g\in G_i}\pi_\theta(\tau_g\mid s_i)\tilde A_{i,g}
\right]
+\beta_r\mathcal{L}_{\mathrm{dist}}
+\beta_f\mathcal{L}_{\mathrm{KL}},
\end{equation}
where $\beta_r$ and $\beta_f$ are the regularization weights in \Cref{tab:implementation_hparams}.

\paragraph{CRAFT.}
The full method keeps the same counterfactual proxy as GRPO, but adds a selected-trajectory update weighted by the corrective reward observed after closed-loop execution. The corrective return is
\begin{equation}
G^{\mathrm{gr}}_t
=
r^{\mathrm{gr}}_t+\gamma_c u^{\mathrm{done}}_tG^{\mathrm{gr}}_{t+1},
\end{equation}
and the implementation scales and clips it as
\begin{equation}
\hat A^{\mathrm{gr}}_t
=
\mathrm{clip}\!\left(\frac{G^{\mathrm{gr}}_t}{S_{\mathrm{gr}}},A_{\min},A_{\max}\right).
\end{equation}
Let $\tau_t$ be the selected trajectory at $s_t$ and $a_t=\kappa(s_t,\tau_t)$ be the executed controller action. The likelihood ratio is $\rho_t=\exp(\log\pi_\theta(\tau_t\mid s_t)-\log\pi_{\mathrm{old}}(\tau_t\mid s_t))$. The grounded residual correction objective uses PPO clipping followed by dual clipping:
\begin{equation}
o_t^{\min}
=
\min\!\left(
\rho_t\hat A^{\mathrm{gr}}_t,\,
\mathrm{clip}(\rho_t,1-\epsilon_{\mathrm{clip}},1+\epsilon_{\mathrm{clip}})\hat A^{\mathrm{gr}}_t
\right),
\end{equation}
\begin{equation}
o_t=
\begin{cases}
\max(o_t^{\min},c\hat A^{\mathrm{gr}}_t), & \hat A^{\mathrm{gr}}_t<0,\\
o_t^{\min}, & \hat A^{\mathrm{gr}}_t\ge 0,
\end{cases}
\end{equation}
The implemented CRAFT loss is
\begin{equation}
\mathcal{L}_{\mathrm{CRAFT}}
=
\lambda_{\mathrm{cp}}\mathcal{L}_{\mathrm{cp}}
+\lambda_{\mathrm{gr}}\mathcal{L}_{\mathrm{gr}}
+\beta_r\mathcal{L}_{\mathrm{dist}}
+\beta_f\mathcal{L}_{\mathrm{KL}},
\end{equation}
where $\lambda_{\mathrm{cp}}$, $\lambda_{\mathrm{gr}}$, $\beta_r$, and $\beta_f$ are set as in \Cref{tab:implementation_hparams}.

\paragraph{Teacher regularization.}
All reported fine-tuners keep a teacher initialized from the pre-trained actor. For the trainable actor subset, the teacher is updated by EMA,
\begin{equation}
\theta_T\leftarrow m\,\theta_T+(1-m)\theta .
\end{equation}
The distillation term is $D_{\mathrm{KL}}(\pi_\theta\|\pi_T)$, and the optional forward-KL term is $D_{\mathrm{KL}}(\pi_T\|\pi_\theta)$. The ego configurations enable the forward-KL term for all reported fine-tuners.

%% file: tab/benchmark_robust.tex
\begin{table}[t]
\centering
\caption{\small \textbf{Evaluation comparison between official Bench2Drive and Bench2Drive$^{*}$.}
Bench2Drive$^{*}$ denotes the patched protocol used for cleaner and more efficient closed-loop training and evaluation. For Bench2Drive$^{*}$, results are reported as mean $\pm$ standard deviation over three random seeds.
}
\label{tab:benchmark_robust}
\resizebox{1.0\textwidth}{!}{
\begin{threeparttable}
\tablestyle{2.0pt}{1.1}
\setlength{\tabcolsep}{1mm}{
\newcolumntype{Y}{>{\centering\arraybackslash}m{13mm}}
\begin{tabular}{c!{\color{black}\vrule width 0.3pt}cYY!{\color{black}\vrule width 0.3pt}cccccc}
    \toprule
    \multicolumn{1}{c}{\multirow{2}{*}{\textbf{Driving Policy}}} & \multicolumn{1}{c}{\multirow{2}{*}{\textbf{Benchmark}}} & \multicolumn{2}{c}{\textbf{Closed-loop Metrics}} & \multicolumn{6}{c}{\textbf{Multi-Ability} (\%)} \\
    \cmidrule(r){3-4}
    \cmidrule(r){5-10}
    \multicolumn{1}{c}{} & \multicolumn{1}{c}{} & DS $\uparrow$ & \multicolumn{1}{Y}{SR (\%) $\uparrow$} & Merging $\uparrow$ & Overtaking $\uparrow$ & Emergency Brake $\uparrow$ & Give Way $\uparrow$ & \multicolumn{1}{c}{Traffic Sign $\uparrow$} & Mean $\uparrow$ \\
    \midrule
    \rowcolor{lightpurple}
    \multicolumn{10}{c}{\textcolor{darkgray}{\textit{Rule-based Planner}}}\\
    \multirow{2}{*}{PDM-Lite~\cite{sima2024drivelm}} & Bench2Drive & 97.02 & 92.27 & - & - & - & - & - & -\\
    & \cellcolor{gray!30}Bench2Drive$^{*}$ & \cellcolor{gray!30}{94.62}\pmsd{0.59} & \cellcolor{gray!30}{88.48}\pmsd{0.69}& \cellcolor{gray!30}{86.30}\pmsd{0.00} & \cellcolor{gray!30}{86.67}\pmsd{2.25} & \cellcolor{gray!30}{98.87}\pmsd{0.98} & \cellcolor{gray!30}{73.33}\pmsd{5.77} & \cellcolor{gray!30}{84.20}\pmsd{2.10} & \cellcolor{gray!30}{85.83}\pmsd{1.25}\\
    \midrule
    \rowcolor{lightpurple}
    \multicolumn{10}{c}{\textcolor{darkgray}{\textit{Privileged Planner}}}\\
    \multirow{2}{*}{PlanTv2~\cite{gerstenecker2025plant2}}  & Bench2Drive & {92.40}\pmsd{1.70} & {83.80}\pmsd{3.30} & - & - & - & - & - & -\\
    & \cellcolor{gray!30}Bench2Drive$^{*}$ & \cellcolor{gray!30}{84.61}\pmsd{0.74} & \cellcolor{gray!30}{76.06}\pmsd{0.69} & \cellcolor{gray!30}{84.20}\pmsd{0.69} & \cellcolor{gray!30}{56.30}\pmsd{2.60} & \cellcolor{gray!30}{88.87}\pmsd{0.98} & \cellcolor{gray!30}{16.67}\pmsd{5.77} & \cellcolor{gray!30}{81.80}\pmsd{1.21} & \cellcolor{gray!30}{65.53}\pmsd{1.35}\\
    \midrule
    \rowcolor{lightpurple}
    \multicolumn{10}{c}{\textcolor{darkgray}{\textit{End-to-End method}}}\\
    \multirow{2}{*}{UniAD~\cite{hu2023_uniad}} & Bench2Drive & 45.81 & 16.36 & - & - & - & - & - & -\\
    & \cellcolor{gray!30}Bench2Drive$^{*}$ & \cellcolor{gray!30}{45.74}\pmsd{0.60} & \cellcolor{gray!30}{20.00}\pmsd{0.79} & \cellcolor{gray!30}{22.90}\pmsd{0.69} & \cellcolor{gray!30}{26.67}\pmsd{6.65} & \cellcolor{gray!30}{15.00}\pmsd{0.00} & \cellcolor{gray!30}{30.00}\pmsd{10.00} & \cellcolor{gray!30}{4.90}\pmsd{1.21} & \cellcolor{gray!30}{19.90}\pmsd{1.01}\\
    \midrule
    \multirow{2}{*}{VAD~\cite{jiang2023vad}} & Bench2Drive & 42.35 & 15.00 & - & - & - & - & - & -\\
    & \cellcolor{gray!30}Bench2Drive$^{*}$ & \cellcolor{gray!30}{44.25}\pmsd{0.67} & \cellcolor{gray!30}{16.06}\pmsd{0.95} & \cellcolor{gray!30}{12.13}\pmsd{1.44} & \cellcolor{gray!30}{5.93}\pmsd{3.42} & \cellcolor{gray!30}{30.53}\pmsd{2.54} & \cellcolor{gray!30}{30.00}\pmsd{10.00} & \cellcolor{gray!30}{2.47}\pmsd{0.64} & \cellcolor{gray!30}{16.23}\pmsd{1.38}\\
    \midrule
    \multirow{2}{*}{SparseDrive~\cite{sun2025sparsedrive}} & Bench2Drive & 44.54 & 16.71 & - & - & - & - & - & -\\
    & \cellcolor{gray!30}Bench2Drive$^{*}$ & \cellcolor{gray!30}{44.21}\pmsd{0.43} & \cellcolor{gray!30}{19.55}\pmsd{0.45} & \cellcolor{gray!30}{14.63}\pmsd{1.44} & \cellcolor{gray!30}{19.30}\pmsd{6.41} & \cellcolor{gray!30}{28.33}\pmsd{2.89} & \cellcolor{gray!30}{36.67}\pmsd{5.77} & \cellcolor{gray!30}{3.53}\pmsd{1.63} & \cellcolor{gray!30}{20.47}\pmsd{0.95}\\
    \bottomrule
\end{tabular}
}
\end{threeparttable}
}
\end{table}

%% file: tab/implementation_hparams.tex
\begin{table*}[t]
\centering
\renewcommand\arraystretch{1.08}
\setlength{\tabcolsep}{4pt}
\small
\caption{\textbf{Main implementation hyperparameters.} The settings are shared across policies unless the table states otherwise. CRAFT uses the counterfactual reward for candidate scoring and the corrective reward for residual correction.}
\label{tab:implementation_hparams}
\resizebox{\textwidth}{!}{
\begin{tabular}{@{}l l l@{}}
\toprule
Component & Setting & Value \\
\midrule
\multirow{14}{*}{Trainer}
& Epochs per update round & 4 \\
& Optimizer & AdamW \\
& Initial / minimum learning rate & $10^{-4}$ / $5{\times}10^{-6}$ \\
& Weight decay & $10^{-5}$ \\
& Gradient clipping & Global norm $0.5$ \\
& Precision / training devices & 32-bit / 4 GPUs \\
& Data decode workers & 8 \\
& Shared rollout discount & $\gamma=0.98$ \\
& GAE trace parameter & $\lambda_{\mathrm{GAE}}=0.95$ \\
& Value clip range & $\epsilon_v=2.0$ \\
& Critic learning-rate multiplier & $\kappa_v=2.0$ \\
& Return-normalization epsilon & $\varepsilon_{\mathrm{norm}}=10^{-8}$ \\
& EMA teacher momentum & $m=0.99$ \\
& Rollout buffer size & $B=2048$ transitions \\
\midrule
\multirow{11}{*}{counterfactual world}
& Tracker & Batched PID trajectory tracker \\
& Virtual time step & $\Delta t=0.1$ s \\
& Rollout horizon & $H=20$ steps, 21 states including the current state \\
& Inference interval & 5 simulator frames \\
& Chunk size & 500 candidate rollouts \\
& Other-agent model & Decay rollout with kinematic bicycle dynamics \\
& Group normalization & Mean/std within valid candidates \\
& Std floor for group advantages & $\sigma_{\min}=5.0$ \\
& Collision speed reference & $v_{\mathrm{ref}}=5.0$ m/s \\
& Collision speed weight & $\alpha_v=0.5$ \\
& Collision speed clip bounds & $[\nu_{\min},\nu_{\max}]=[0,1]$ \\
\midrule
\multirow{9}{*}{Fine-tuning objectives}
& Counterfactual proxy loss weight $\lambda_{\mathrm{cp}}$ & 1.0 \\
& Grounded residual loss weight $\lambda_{\mathrm{gr}}$ & 0.5 \\
& Distillation weight $\beta_r$ & 0.5 \\
& Forward-KL weight $\beta_f$ & 0.1 \\
& PPO clip range $\epsilon_{\mathrm{clip}}$ & 0.2 \\
& Dual-clip constant $c$ & 2.0 \\
& Corrective discount $\gamma_c$ & 0.8 \\
& Corrective return scale & $S_{\mathrm{gr}}=8.0$ \\
& Corrective advantage clip & $[A_{\min},A_{\max}]=[-1,1]$ \\
\bottomrule
\end{tabular}
}
\end{table*}

%% file: tab/reward_hparams.tex
\begin{table*}[t]
\centering
\renewcommand\arraystretch{1.05}
\setlength{\tabcolsep}{4pt}
\scriptsize
\caption{\textbf{Reward hyperparameters.} Superscripts distinguish the counterfactual proxy reward from the grounded residual corrective reward; these instantiate the dense counterfactual proxy and grounded residual correction, respectively.}
\label{tab:reward_hparams}
\resizebox{\textwidth}{!}{
\begin{tabular}{@{}l l l l@{}}
\toprule
Reward & Symbol & Value & Role \\
\midrule
\multirow{16}{*}{Counterfactual reward}
& $p_{\min}$ / $p_{\max}$ & 0.0 / 1.2 & Progress clipping bounds \\
& $w_{\mathrm{prog}}$ & 8.0 & Progress weight \\
& $w_g$ / $w_c$ / $w_h$ & 3.0 / 0.8 / 2.0 & Route, centerline, and heading efficiency weights \\
& $\eta_{\min}$ & 0.0 & Minimum route-efficiency multiplier \\
& $d_{\mathrm{clip}}$ & 0.10 & Deviation-delta clipping bound \\
& $\delta_g^{\mathrm{rec}}$ / $\delta_c^{\mathrm{rec}}$ & 0.08 / 0.08 & Recovery activation thresholds \\
& $k_g$ / $k_c$ / $k_h$ & 0.4 / 0.2 / 0.4 & Recovery weights \\
& $c_{\mathrm{clip}}$ & 0.5 & Recovery reward clipping bound \\
& $a_{\mathrm{rec}}$ / $b_{\mathrm{rec}}$ & 0.5 / 0.5 & Progress-dependent recovery scaling \\
& $\lambda^{\mathrm{cp}}_{\mathrm{offroad}}$ & 1.5 & Off-road step cost \\
& $\lambda^{\mathrm{cp}}_{\mathrm{opp}}$ & 0.1 & Opposite-lane step cost \\
& $\lambda^{\mathrm{cp}}_{\mathrm{offroute}}$ & 1.5 & Off-route step cost \\
& $\lambda^{\mathrm{cp}}_{\mathrm{emg}}$ & 1.0 & Emergency-lane step cost \\
& $\lambda^{\mathrm{cp}}_{\mathrm{coll}}$ & 40.0 & First-collision penalty \\
& $\lambda^{\mathrm{cp}}_{\mathrm{red}}$ & 40.0 & Red-light trajectory penalty \\
& $\lambda^{\mathrm{cp}}_{\mathrm{stop}}$ & 40.0 & Stop-sign trajectory penalty \\
\midrule
\multirow{8}{*}{Corrective reward}
& $\lambda^{\mathrm{gr}}_{\mathrm{offroad}}$ & 0.5 & Off-road penalty \\
& $\lambda^{\mathrm{gr}}_{\mathrm{emg}}$ & 0.2 & Emergency-lane penalty \\
& $\lambda^{\mathrm{gr}}_{\mathrm{offroute}}$ & 0.5 & Off-route penalty \\
& $\lambda^{\mathrm{gr}}_{\mathrm{red}}$ & 2.0 & Red-light violation penalty \\
& $\lambda^{\mathrm{gr}}_{\mathrm{stop}}$ & 2.0 & Stop-sign violation penalty \\
& $\lambda^{\mathrm{gr}}_{\mathrm{coll}}$ & 5.0 & Collision penalty \\
& $v_{\mathrm{stop}}$ & 0.1 m/s & Stop-required motion threshold \\
& $v_{\mathrm{go}}$ & 2.0 m/s & Go-required slow-motion threshold \\
\bottomrule
\end{tabular}
}
\end{table*}